\documentclass[]{article}
\usepackage{geometry}
\usepackage{algorithm}
\usepackage{algpseudocode}

\usepackage{amsmath}
\usepackage{amssymb}

\usepackage{bigfoot}

\usepackage{enumerate}

\usepackage{fancyhdr}
\usepackage{fancyvrb}

\usepackage[hidelinks]{hyperref} 

\usepackage{amsthm} 
\newtheorem{theorem}{Theorem}[section]
\newtheorem{lemma}[theorem]{Lemma}
\newtheorem{corollary}[theorem]{Corollary}

\newtheorem{definition}[theorem]{Definition}
\newtheorem{proposition}[theorem]{Proposition}

\newtheorem{remark}[theorem]{Remark}

\newtheorem{example}[theorem]{Example}

\newcommand{\dd}{\mathrm{d}}

\usepackage{mathrsfs}
\usepackage{physics}

\usepackage{xcolor} 

\usepackage{caption}
\usepackage{subcaption}
\usepackage{graphicx}

\title{Physics-informed reduced-order modelling with equivariant spectral submanifolds}

\author{%
    Georg Maierhofer\\
    \small Department of Applied Mathematics and Theoretical Physics\\
    \small University of Cambridge, UK\\
    \small \texttt{gam37@cam.ac.uk}
}

\date{}

\usepackage{xcolor}

\usepackage[toc,page]{appendix}

\usepackage{enumerate}

\usepackage{bm}

\usepackage{mathdots}

\renewcommand{\Re}{\mathrm{Re}}
\renewcommand{\Im}{\mathrm{Im}}

\usepackage{enumitem}
\usepackage{booktabs}

\begin{document}
\maketitle
\begin{abstract}
  \noindent Spectral submanifold (SSM) reduction has emerged as a mathematically principled route to reliable nonlinear reduced-order models, capturing dynamics beyond the reach of linear techniques such as Dynamic Mode Decomposition (DMD). The computation of SSMs, however, remains computationally expensive, particularly for high-dimensional systems. In this work, we introduce equivariant spectral submanifold (eSSM) reduction, a novel extension of the SSM framework that explicitly incorporates symmetries of the full-order model into the reduction process. We establish the mathematical foundations of this approach by showing that SSMs are naturally equivariant submanifolds and that the associated charts and reduced dynamics inherit the appropriate induced group actions. Building on this framework, we develop a novel equivariant SSM reduction algorithm that exploits these symmetries to achieve substantially faster computations while also improving model robustness. We demonstrate the advantages of this approach on several benchmark problems including a test from the Common Task Framework for Science.
\end{abstract}

modelling


\section{Introduction}
Reduced order models are of paramount importance in many scientific disciplines facilitating the simulation of the main features of complex dynamics in a computationally efficient manner \cite{brenneretal2017}. There are a range of techniques that can be applied for this task, including linear methods such as dynamic mode decomposition (DMD) \cite{SCHMID_2010, tu2014_koopman, kutz_dmd_book2016}, as well as nonlinear methods such as the Koopman operator framework \cite{bruntonkutzkoopmanreview22,Colbrook2016} and spectral submanifold (SSM) reduction \cite{haller2016nonlinear,haller2023nonlinear}. Such reduced-order methods are particularly useful in the context of data-driven modelling, where they can provide a principled and interpretable framework for modelling and predicting complex dynamics often outperforming other types of data-driven models in terms of predictive accuracy and generalisation \cite{riva2026ctf4nuclear,wyder2026common,yermakov2026the}.

In particular, SSMs have recently emerged as a powerful tool for model reduction of nonlinear dynamical systems. SSMs are invariant manifolds that serve as the smoothest nonlinear continuation of spectral subspaces of the linearised system around a fixed point. As such, SSMs provide a natural and rigorous framework for reduced-order modelling that incorporates the nonlinear aspect of the dynamics without having to resort to linearisation or choose ad-hoc observables. The existence of a unique SSM was first proved by Haller and Ponsioen \cite{haller2016nonlinear} using the parametrisation method of Cabré et al. \cite{Cabre2003a,Cabre2003b,Cabre2005} in the case of stable spectral subspaces and later extended to mixed internal
stability types in \cite{haller2023nonlinear}. Since then, there have been several successful developments of capable and versatile toolboxes for the efficient computation of SSMs and their associated reduced dynamics. In particular, SSMTool \cite{jain2022compute,shobhit_jain_2021_4614202} is a powerful Matlab toolbox for the computation of SSMs in mechanical systems, while the data-driven SSMLearn reduction method introduced by \cite{SSMLearnpaper22,FastSSMpaper23} allows the discovery of SSM reduced models directly from data. However, several challenges remain including the cost of the optimisation problems associated with finding parametrisations of the spectral submanifold and the reduced dynamics which grows with reduced order dimension and the polynomial degree in the reduction algorithm, as well as the incorporation of physical symmetries in the reduced models to faithfully reflect underlying physical principles.

The incorporation of such physical symmetries has been a long-studied problem in the context of numerical methods \cite{mclachlan2024functional,hairer2013geometric,maierhofer_schratz_24,bronsard2026symmetric,feng2025explicit,faou2026fully} and has also seen increasing attention in the context of data-driven modelling with the idea of creating more physically consistent and robust models learned from data. In particular, in the context of DMD recent work by \cite{baddoo2023physics} has introduced a physics-informed DMD (piDMD) method that incorporates symmetries into the DMD framework. In symbolic regression \cite{sindy,rudy2017data}, symmetries and conservation laws are treated in \cite{Yang2024} and, in the context of Koopman operator learning, this has been explored in recent work by \cite{salova2019koopman,HARDER2025134725}.

Notably, structure-preserving reduced-order modelling remains largely
unexplored in the context of SSMs. A first step in this direction is the recent work of \cite{Kaszas_Haller_2024}, who exploited the shift-reflect symmetry of pipe flow by restricting the dynamics to a symmetry-invariant subspace, in which the edge state becomes amenable to SSM-based reduction. In that approach, however, the symmetry serves only to constrain the ambient dynamics prior to reduction: the SSM parametrisation and the reduced dynamics are still computed without reference to the group action, so the reduction step itself does not benefit computationally from the symmetry.

In this work we fill this gap by showing how linear symmetries of the
full-order model can be incorporated directly into the SSM reduction process. We prove that SSMs of equivariant systems are themselves equivariant submanifolds, and that suitably chosen charts and the associated reduced dynamics inherit induced actions of the symmetry group. Building on these results, we develop an equivariant SSM reduction algorithm (eSSM) that constrains both the manifold parametrisation and the reduced dynamics to symmetry-adapted coefficient spaces. This yields significantly faster computations as well as higher model
robustness, since the learned model cannot drift away from the symmetric dynamics of the full system.

The remainder of this manuscript is structured as follows. We begin in \S\ref{sec:equivariant_ssms} by reviewing the basic theory of SSMs before developing their equivariance properties. In \S\ref{sec:computing_ssms} we then discuss the computation of SSMs and the associated reduced dynamics as well as associated equivariance properties, before introducing our novel eSSM reduction algorithm in \S\ref{sec:data_driven_ssm_reduction_with_equivariance}. We demonstrate the advantages of this approach on several benchmark problems in \S\ref{sec:examples}, including a test problem from the Common Task Framework for Science. Finally, concluding remarks are provided in \S\ref{sec:conclusions}.

The code associated with this manuscript is publicly available as an installable Python package \texttt{eSSM} at \url{https://github.com/GeorgAUT/eSSM}.

\section{Spectral submanifolds and equivariance}\label{sec:equivariant_ssms}
In this work we consider the following general form of nonlinear dynamical systems around a given fixed point $\mathbf{x}=\mathbf{0}$:
\begin{align}\label{eqn:general_system}
\dot{\mathbf{x}} = \mathbf{A}\mathbf{x} + \mathbf{f}(\mathbf{x}), \quad \mathbf{x} \in \mathbb{R}^n,
\end{align}
where $\mathbf{A}$ is a linear operator (the linearisation of the system at $\mathbf{x}=\mathbf{0}$) and $\mathbf{f}=\mathcal{O}(\|\mathbf{x}\|^2)$ is a smooth nonlinear function $\mathbf{f}\in C^\infty(\mathbb{R}^n, \mathbb{R}^n)$. We assume that its fixed point at $\mathbf{x}=\mathbf{0}$ is hyperbolic, i.e. the spectrum $\operatorname{Spect}(\mathbf{A})$ does not intersect the imaginary axis, $0 \notin \Re \left[\operatorname{Spect}(\mathbf{A})\right]$.

Spectral submanifolds are a nonlinear extension of the concept of spectral subspaces of the linearised system. Let us denote the eigenvalues of $\mathbf{A}$ by $\lambda_j=\alpha_j + i \omega_j, j=1,\ldots,n,$ ordered by
\begin{align*}
\Re \lambda_1 \leq \Re \lambda_2 \leq \ldots \leq \Re \lambda_n.
\end{align*}
Since $\mathbf{A}$ is a real matrix, its eigenvalues are either real or come in complex conjugate pairs, and we denote the corresponding eigenvectors by $\mathbf{v}_j$.
\begin{definition}
  The real modal eigenspaces $E_j$ are defined as the real span of the eigenvectors associated with $\lambda_j$. In particular, if $\lambda_j$ is a real eigenvalue, then $E_j = \operatorname{span}_{\mathbb{R}} \{\mathbf{v}_j\}$, while if $\lambda_j$ is a complex eigenvalue, then $E_j = \operatorname{span}_{\mathbb{R}} \{\mathbf{v}_j, \overline{\mathbf{v}}_j\}$.
\end{definition}
The modal eigenspaces $E_j$ are invariant under the linear flow generated by $\mathbf{A}$. A direct sum of modal eigenspaces, $E=\bigoplus_{j \in J} E_j$, is called a spectral subspace and is also invariant under the flow generated by $\mathbf{A}$. For a given spectral subspace $E$, we can ask whether there exists an invariant manifold of the full nonlinear system \eqref{eqn:general_system} that is tangent to $E$ at the fixed point $\mathbf{x}=\mathbf{0}$ and that captures the nonlinear dynamics associated with the modes in $E$. Based purely on this tangency requirement there are infinitely many such invariant manifolds, however, \cite{haller2023nonlinear,haller2016nonlinear} proved that under certain non-resonance conditions, there exists a unique invariant manifold $\mathcal{W}(E)$ that is as smooth as the full system and that is tangent to $E$ at the fixed point. This ``smoothest'' invariant manifold is called the spectral submanifold (SSM) associated with the spectral subspace $E$. The existence of this SSM is guaranteed under certain non-resonance conditions on the eigenvalues of $\mathbf{A}$. These resonance conditions are common practice in the relevant literature and generally hold for generic parameter configurations of typical dissipative systems, however, they may fail for simple toy examples with non-generic parameter choices.

\begin{definition}[Global non-resonance condition]
\label{def:global-nonresonance}
We say that $\mathbf{A}$ satisfies the
{global non-resonance condition} if, with the exception of possible $1{:}1$ resonances created by repeated eigenvalues, for any $j\in\{1,\dots,n\}$ and any $\mathbf{m}=(m_1,\dots,m_n)\in\mathbb{N}^{n}$ with $|\mathbf{m}|:=\sum_{k=1}^{n} m_k \ge 2$ we have
\begin{align*}
  \lambda_j \;\neq\; \sum_{k=1}^{n} m_k\,\lambda_k.
\end{align*}
\end{definition}

Under this global non-resonance condition, we then have the following existence result proved by \cite{haller2023nonlinear} adapted to our setting.
\begin{theorem}[{Theorem 1 in \cite{haller2023nonlinear}}]\label{thm:existence-ssm}
Assume that $\mathbf{A}$ is semi-simple and satisfies the global non-resonance condition of Definition \ref{def:global-nonresonance} and let $E$ be a spectral subspace of $\mathbf{A}$. Then there is a unique invariant manifold $\mathcal{W}(E)$ of class $C^\infty$ and dimension $\operatorname{dim} E$ that is tangent to $E$ at $\mathbf{x}=\mathbf{0}$. Moreover, $\mathcal{W}(E)$ admits, near the origin, a local representation as a graph over $E$ with integer-powered Taylor expansion.
\end{theorem}

\begin{remark}
  Note there are several versions of this result which relax various of the assumptions and prove existence of a unique smoothest invariant manifold in further generality, for example without the global non-resonance conditions (provided the associated spectral subspace is stable or unstable). There are also extensions for non-autonomous systems. In the interest of clarity we have restricted the current presentation to the simplest version of the result, and the interested reader is referred to the original papers \cite{haller2023nonlinear,haller2016nonlinear} and follow-on works for further details.
\end{remark}

\subsection{Linear symmetries and equivariance}
For completeness let us recall the definition of equivariance for a dynamical system. We will, throughout this work, focus on linear symmetries that fix the origin, i.e. $S\mathbf{0} = \mathbf{0}$ for all $S \in \mathcal{G}$.
\begin{definition}
  The dynamical system \eqref{eqn:general_system} is said to be \textit{equivariant} with respect to a linear symmetry group $\mathcal{G}\subset GL(n,\mathbb{R})$ acting on $\mathbb{R}^n$ if for all $S \in \mathcal{G}$ and $\mathbf{x} \in \mathbb{R}^n$ we have
  \begin{align}\label{eqn:equivariance_condition}
    S \left(A\mathbf{x} + \mathbf{f}(\mathbf{x})\right) = A S \mathbf{x} + \mathbf{f}(S\mathbf{x}).
  \end{align}
\end{definition}

Such linear symmetries naturally arise in many applications, for example as a result of spatial symmetries in the underlying domain (e.g.\ evolution of waves on periodic or spherical domains) and intrinsic symmetries of the dynamical equations themselves, such as permutation symmetries of coupled identical units. The following is a simple guiding example of such a system with a linear symmetry that we will use as a running example throughout the paper, further examples are provided in \S\ref{sec:examples}.

\begin{example}\label{ex:chain_of_oscillators}
Our first example is a damped oscillator chain with $\ell$ masses and an additional nonlinear spring attached to the leftmost mass, cf. Figure \ref{fig:oscillators}, similarly to the basic setup provided in \cite[\S 4.1]{axaas2023model}. The equations of motion are
\begin{align*}
  \mathbf{M}\ddot{\mathbf{q}} + \mathbf{C}\dot{\mathbf{q}} + \mathbf{K}\mathbf{q} + \mathbf{f}_{\mathrm{nl}}(\mathbf{q},\dot{\mathbf{q}}) = \mathbf{0},
\end{align*}
where $\mathbf{q} = (q_1, \ldots, q_\ell)^\top$ is the vector of displacements from equilibrium. The mass matrix is diagonal,
\begin{align*}
  \mathbf{M} = \operatorname{diag}(m_1, m, \ldots, m), \qquad m_1 = 1.5,\ m = 1,
\end{align*}
and the stiffness matrix is the standard tridiagonal coupling matrix with a free right end,
\begin{align*}
  \mathbf{K} = k\begin{pmatrix}
    2 & -1 &        &        &    \\
    -1 & 2 & -1     &        &    \\
       & \ddots & \ddots & \ddots &    \\
       &        & -1     & 2      & -1 \\
       &        &        & -1     & 1
  \end{pmatrix}, \qquad k = 1.
\end{align*}
Damping is of Rayleigh type, $\mathbf{C} = \alpha \mathbf{M} + \beta \mathbf{K}$, with $\alpha = 2\cdot 10^{-3}$ and $\beta = 5\cdot 10^{-3}$. The nonlinear spring acts only on the leftmost mass and depends on both its displacement and velocity:
\begin{align*}
  \mathbf{f}_{\mathrm{nl}}(\mathbf{q},\dot{\mathbf{q}}) = m_1 \bigl( 3\,q_1^3 + 7\,q_1^4\,\dot q_1 + 5\,\dot q_1^{\,3} \bigr)\, \mathbf{e}_1,
\end{align*}
where $\mathbf{e}_1$ is the first standard basis vector. Introducing the state $\mathbf{x} = (\mathbf{q}^\top, \dot{\mathbf{q}}^\top)^\top \in \mathbb{R}^{2\ell}$, the system can be recast in first-order form as
\begin{align*}
  \dot{\mathbf{x}} = \mathbf{A}\mathbf{x} + \mathbf{F}(\mathbf{x}), \qquad
  \mathbf{A} = \begin{pmatrix} \mathbf{0} & \mathbf{I} \\ -\mathbf{M}^{-1}\mathbf{K} & -\mathbf{M}^{-1}\mathbf{C} \end{pmatrix}, \qquad
  \mathbf{F}(\mathbf{x}) = \begin{pmatrix} \mathbf{0} \\ -\mathbf{M}^{-1}\mathbf{f}_{\mathrm{nl}}(\mathbf{q},\dot{\mathbf{q}}) \end{pmatrix}.
\end{align*}

We note that the system is equivariant with respect to the spatial symmetry group $\mathcal{G} = \langle -\mathbf{I}_{2\ell} \rangle = \{\mathbf{I}_{2\ell}, -\mathbf{I}_{2\ell}\} \cong C_2$ (where $C_2$ denotes the cyclic group of order 2), which acts on the phase space $\mathbb{R}^{2\ell}$ by negation. In other words, the system is invariant under the transformation $\mathbf{x} \mapsto -\mathbf{x}$, which corresponds to simultaneously reversing the direction of all displacements and velocities and can easily be verified by noting that $\mathbf{F}(\mathbf{x})$ contains only odd terms.

\begin{figure}[h!]
  \centering
  \includegraphics[width=0.8\textwidth]{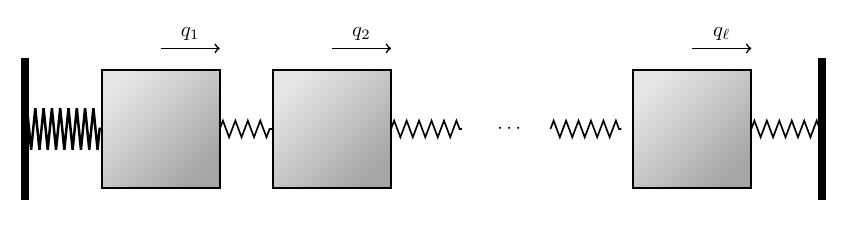}
  \caption{Schematic of the chain of oscillators, with nonlinear force at left end.}
  \label{fig:oscillators}
\end{figure}
\end{example}

\subsection{Equivariance of SSMs}
The uniqueness of the SSM $\mathcal{W}(E)$ in Theorem~\ref{thm:existence-ssm} is especially useful for the treatment of equivariant systems. In particular, we can immediately deduce the following result about the equivariance of SSMs.
\begin{theorem}\label{thm:equivariance-ssm} Suppose that the system \eqref{eqn:general_system} is equivariant with respect to a linear symmetry group $\mathcal{G}$ and let $E$ be a spectral subspace of $\mathbf{A}$. Then the following statements hold:
  \begin{enumerate}[label=(\roman*)]
    \item the spectral subspace $E$ is invariant under the action of $\mathcal{G}$, i.e. for all $S \in \mathcal{G}$ we have $SE = E$;
    \item the unique smoothest SSM $\mathcal{W}(E)$ associated with $E$ is also invariant under the action of $\mathcal{G}$, i.e. for all $S \in \mathcal{G}$ we have $S\mathcal{W}(E) = \mathcal{W}(E)$.
  \end{enumerate}
\end{theorem}
\begin{proof} Let $S\in \mathcal{G}$ be a linear symmetry of the system. Let us differentiate the equivariance condition \eqref{eqn:equivariance_condition} with respect to $\mathbf{x}$ at $\mathbf{x}=\mathbf{0}$ to obtain (since $\mathbf{f}(\mathbf{0}) = \mathbf{0}$ and $D\mathbf{f}(\mathbf{0}) = \mathbf{0}$):
\begin{align*}
  S\mathbf{A} = \mathbf{A}S.
\end{align*}
Subtracting from \eqref{eqn:equivariance_condition} implies that for any $\mathbf{x} \in \mathbb{R}^n$ we also have
  \begin{align*}
 S \mathbf{f}(\mathbf{x}) = \mathbf{f}(S\mathbf{x}).
  \end{align*}
  This implies that $S$ commutes with the linear flow generated by $\mathbf{A}$, which in turn implies (together with invertibility of $S$) that $SE = E$ for any spectral subspace $E$ of $\mathbf{A}$.
  
  It remains to show that the unique smoothest SSM $\mathcal{W}(E)$ is also invariant under the action of $S$. Let $\varphi_{t}:\mathbb{R}^n \to \mathbb{R}^n$ be the flow map of the system \eqref{eqn:general_system}. Since the system is equivariant with respect to $\mathcal{G}$, we have that for all $S \in \mathcal{G}$ and $\mathbf{x} \in \mathbb{R}^n$,
\begin{align*}
  \frac{d}{dt} S \varphi_{t}(\mathbf{x}) = S (\mathbf{A}\mathbf{x}+\mathbf{f}(\mathbf{x}))=\mathbf{A}S\mathbf{x}+\mathbf{f}(S\mathbf{x}) = \frac{d}{dt} \varphi_{t}(S\mathbf{x}).
\end{align*}
Thus, by smoothness of \eqref{eqn:general_system} and uniqueness of solutions, the flow-map is $S$-equivariant, i.e. $S\varphi_{t}(\mathbf{x})=\varphi_t(S\mathbf{x})$. Let us now consider $S(\mathcal{W}(E))$. Since $S$ is linear and invertible, this is $\mathcal{C}^{\infty}$-manifold with the following properties:
\begin{enumerate}[label=(\roman*)]
  \item $\mathbf{0} = S\mathbf{0} \in S\mathcal{W}(E)$, since $\mathbf{0}\in\mathcal{W}(E)$;
  \item $T_{\mathbf{0}}S\mathcal{W}(E) = S\,T_{\mathbf{0}}\mathcal{W}(E) = SE = E$, since $T_{\mathbf{0}}\mathcal{W}(E) = E$ by Theorem~\ref{thm:existence-ssm} and $SE = E$ by the first part of this proposition;
  \item for any $\mathbf{y}=S\mathbf{x}\in S\mathcal{W}(E)$ with $\mathbf{x}\in\mathcal{W}(E)$,
  \begin{align*}
    \varphi_t(\mathbf{y}) = \varphi_t(S\mathbf{x}) = S\varphi_t(\mathbf{x})\in S\mathcal{W}(E),
  \end{align*}
  thus $S\mathcal{W}(E)$ is invariant under the flow $\varphi_t$.
\end{enumerate}
By uniqueness in Theorem~\ref{thm:existence-ssm} we thus must have $S\mathcal{W}(E) = \mathcal{W}(E)$, which completes the proof.
\end{proof}

\section{Computing spectral submanifolds and reduced dynamics on SSMs}\label{sec:computing_ssms} The central pillar of SSM-based reduced order modelling, as first introduced by \cite{jain2022compute}, is to find a parametrisation of the SSM $\mathcal{W}(E)$ as a graph over the spectral subspace $E$ and to then obtain the reduced dynamics on $\mathcal{W}(E)$ by restriction. The standard construction of SSM-based reduced order models \cite{jain2022compute} takes the parametrisation of the SSM along the orthogonal complement $E^\perp$. We will see that in the equivariant setting it is advantageous to allow the construction of this parametrisation along a more general complement of $E$ that is not necessarily orthogonal. Thus, let us begin by fixing a complement $E^c$ of $E$, i.e. a subspace such that $\mathbb{R}^n = E \oplus E^c$. Note this choice of $E^c$ is not unique, and we will discuss two natural choices for $E^c$ in the equivariant setting in \S\ref{sec:equivariant_ssm_reduction} before fixing the choice $E^c$ to the most efficient choice for the rest of the manuscript. A similar ``oblique projection'' construction was considered in the context of strongly non-normal $\mathbf{A}$ in \cite{bettini2025data}, however the use of this in the equivariant setting has not previously been explored. Throughout this section we will focus on compact symmetry groups $\mathcal{G}$.

Let $d=\mathrm{dim}(\mathcal{W}(E))=\mathrm{dim}(E)$. We begin by parametrising $E$ and $E^c$ with a suitable basis $\mathbf{U}_1=\left(\mathbf{u}_1 \ldots \mathbf{u}_d\right)\in \mathbb{R}^{n \times d}$ and $\mathbf{U}_2=\left(\mathbf{u}_{d+1} \ldots \mathbf{u}_n\right)\in \mathbb{R}^{n \times (n-d)}$ respectively, such that $[\mathbf{U}_1 \mid \mathbf{U}_2]\in\mathbb{R}^{n\times n}$ is invertible. The inverse then defines the dual basis $\mathbf{V}_1 \in\mathbb{R}^{n \times d},\mathbf{V}_2 \in\mathbb{R}^{n \times (n-d)}$ given by
\begin{align*}
  \begin{pmatrix} \mathbf{V}_1^\top \\ \mathbf{V}_2^\top \end{pmatrix} = [\mathbf{U}_1 \mid \mathbf{U}_2]^{-1},
\end{align*}
which satisfies $\mathbf{V}_1^\top \mathbf{U}_1 = \mathbf{I}_d$, $\mathbf{V}_2^\top \mathbf{U}_2 = \mathbf{I}_{n-d}$, $\mathbf{V}_1^\top \mathbf{U}_2 = 0$ and $\mathbf{V}_2^\top \mathbf{U}_1 = 0$. This means that the projection $\pi_{E}$ onto $E$ along $E^c$ is given by $\pi_{E} = \mathbf{U}_1 \mathbf{V}_1^\top$, while the projection $\pi_{E^c}$ onto $E^c$ along $E$ is given by $\pi_{E^c} = \mathbf{U}_2 \mathbf{V}_2^\top$.
\begin{remark}
  In the case $E^c = E^\perp$ the primal basis $\mathbf{U}_1$ can be chosen as orthonormal, in which case the dual basis is in fact the same as the original basis, i.e. $\mathbf{V}_1 = \mathbf{U}_1$ and $\mathbf{V}_2 = \mathbf{U}_2$, since the columns of $[\mathbf{U}_1 \mid \mathbf{U}_2]$ are orthonormal, thus reducing what follows to the standard construction of SSM-based reduced order models as in \cite{jain2022compute}.
\end{remark}

The reduced coordinates $\boldsymbol{\eta}$ on $E$ can then be obtained from a given point $\mathbf{y}\in\mathcal{W}(E)$ by projection along $E^c$ (cf. Figure~\ref{fig:graph_parametrisation}), i.e.
  $\boldsymbol{\eta} = \mathbf{V}_1^\top \mathbf{y}$. By Theorem~\ref{thm:existence-ssm} we can then, in a neighbourhood of $\mathbf{0}$, find a parametrisation of $\mathcal{W}(E)$ as a graph over $E$ with integer-powered Taylor expansion, i.e.\ (for a fixed choice of $E^c$) there exists a unique function $\mathbf{h}: E \rightarrow \mathbb{R}^n$ such that $\mathbf{h}(\mathbf{0}) = \mathbf{0}$, $\mathbf{V}_1^\top \mathbf{h}(\boldsymbol{\eta}) = \mathbf{0}$ for all $\boldsymbol{\eta} \in \mathbb{R}^d$, and $D_{\boldsymbol{\eta}}\mathbf{h}(\mathbf{0}) = \mathbf{0}$ such that
\begin{align}\label{eqn:parametrisation_WE}
  \mathbf{y} = \mathbf{U}_1\boldsymbol{\eta} + \mathbf{h}(\boldsymbol{\eta}).
\end{align}
The reduced dynamics on $E$ are then given by 
\begin{align}\label{eqn:reduced_dynamics}
\dot{\boldsymbol{\eta}}=\mathbf{r}(\boldsymbol{\eta}) = \mathbf{V}_1^\top \left( \mathbf{A}\mathbf{U}_1\boldsymbol{\eta} + \mathbf{A}\mathbf{h}(\boldsymbol{\eta}) + \mathbf{f}(\mathbf{U}_1\boldsymbol{\eta} + \mathbf{h}(\boldsymbol{\eta})) \right).
\end{align}\vspace{-0.5cm}
\begin{figure}[h!]
  \centering
  \includegraphics[width=0.7\textwidth]{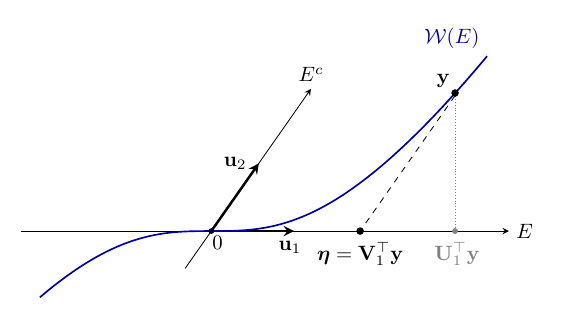}
  \caption{Projection onto $E$ along $E^c$.}
  \label{fig:graph_parametrisation}
\end{figure}
In principle, if $\mathbf{f}$ is analytic, then $\mathbf{h}$ and $\mathbf{r}$ can be written as a convergent series expansion in $\boldsymbol{\eta}$ (cf. \cite{haller2023nonlinear}). However, generically, the expansion of the vector field $\mathbf{r}$ contains coefficients that are messy and partly redundant artefacts of the particular choice of representation of the system \eqref{eqn:general_system}. Thus \cite{jain2022compute,SSMLearnpaper22} advocate for the representation of the reduced dynamics in an extended Poincar\'e normal form style (cf. \cite{poincare1892methodes1,arnold1983geometrical,guckenheimer1983nonlinear}) to promote model sparsity and help develop efficient and robust algorithms for computing the reduced dynamics from data. For this we seek a nonlinear change of coordinates $\boldsymbol{\eta}=\mathbf{t}(\mathbf{z})$ such that the transformed vector field $\dot{\mathbf{z}} = \mathbf{n}(\mathbf{z})$ has a diagonal linear part and as few nonlinear terms in its Taylor expansion as possible. The ``extended'' part of the normal form refers to the fact that we do not seek to remove all non-resonant terms from \eqref{eqn:reduced_dynamics}. Following \cite{SSMLearnpaper22} we do not remove the non-resonant terms that would lead to small denominators in the Taylor expansion of $\mathbf{t}$. This means effectively that near-resonant terms are assigned to the reduced dynamics and all other terms to the change of coordinates. Specifically, we let $\mathbf{B}:= \mathbf{V}_1^\top\mathbf{A}\mathbf{U}_1 \in \mathbb{R}^{d\times d}$, i.e. $\mathbf{B}$ is the linearisation of the reduced dynamics in $E$ around $\mathbf{0}$, and we seek a change of coordinates $\mathbf{t}:\mathbb{C}^d \rightarrow \mathbb{C}^d$ such that the transformed vector field $\mathbf{n}:\mathbb{C}^d \rightarrow \mathbb{C}^d,$ given by
\begin{align}\label{eqn:extended_normal_form}
  \dot{\mathbf{z}} = \mathbf{n}(\mathbf{z}), \quad \boldsymbol{\eta} = \mathbf{t}(\mathbf{z}),
\end{align}
has a diagonal linear part and as few nonlinear terms in its Taylor expansion as possible. Since the reduced coordinates $\boldsymbol{\eta}$ are real, we require that the change of coordinates $\mathbf{t}$ maps the invariant subspace $\{\mathbf{z} : z_{k} = \bar{z}_{k'} \text{ for conjugate pairs } (k,k')\}$ into $\mathbb{R}^d$ and that $\mathbf{n}$ is real on this invariant subspace. We diagonalise
\begin{align}\label{eqn:def_B_and_Lambda}
  \mathbf{B} = \mathbf{W}\boldsymbol{\Lambda}\mathbf{W}^{-1}, \quad \boldsymbol{\Lambda} = \operatorname{diag}(\lambda_1,\dots,\lambda_d),
\end{align}
where the eigenvalues and the columns of $\mathbf{W}\in\mathbb{C}^{d\times d}$ come in complex conjugate pairs (and the eigenvalues correspond to the eigenvalues of $\mathbf{A}$ associated with the spectral subspace $E$). Our new coordinates and normal form are then sought as expansions
\begin{align*}
  \mathbf{t}(\mathbf{z}) = \mathbf{W}\mathbf{z}
    + \sum_{|\mathbf{m}|\geq 2}\mathbf{t}_{\mathbf{m}}\mathbf{z}^{\mathbf{m}},
  \qquad
  \mathbf{n}(\mathbf{z}) = \boldsymbol{\Lambda}\mathbf{z}
    + \sum_{|\mathbf{m}|\geq 2}\mathbf{n}_{\mathbf{m}}\mathbf{z}^{\mathbf{m}},
\end{align*}
subject to the change of coordinates condition between the reduced dynamics and the normal form
\begin{align}\label{eqn:conjugacy}
  D\mathbf{t}(\mathbf{z})\,\mathbf{n}(\mathbf{z}) = \mathbf{r}(\mathbf{t}(\mathbf{z})).
\end{align}
Solving \eqref{eqn:conjugacy} order by order in $\mathbf{z}$ yields the homological equations
\begin{align}\label{eqn:homological}
  \left(\mathbf{m}\cdot\boldsymbol{\lambda} - \lambda_j\right)
  \widetilde{t}_{j,\mathbf{m}} + n_{j,\mathbf{m}} = g_{j,\mathbf{m}},
\end{align}
where $\widetilde{\mathbf{t}} := \mathbf{W}^{-1}\mathbf{t}$ and $g_{j,\mathbf{m}}$ collects terms arising from \eqref{eqn:reduced_dynamics} and from lower-order coefficients. At each order, \eqref{eqn:homological} is a single scalar equation in the two unknowns $\widetilde{t}_{j,\mathbf{m}}$ and $n_{j,\mathbf{m}}$, so every monomial may be assigned to exactly one of the two maps: removing it from the reduced dynamics ($n_{j,\mathbf{m}} = 0$, $\widetilde{t}_{j,\mathbf{m}} = g_{j,\mathbf{m}}/(\mathbf{m}\cdot\boldsymbol{\lambda}-\lambda_j)$) is possible whenever $\mathbf{m}\cdot\boldsymbol{\lambda} \neq \lambda_j$, but produces a small denominator whenever the system is near an inner resonance. The extended normal form formulation in \cite{SSMLearnpaper22} therefore retains these terms in the reduced dynamics by fixing a threshold $\delta > 0$ and defining the near-resonant index set
\begin{align*}
  \mathcal{I}_\delta := \left\{ (j,\mathbf{m}) :
  |\mathrm{Im}(\mathbf{m}\cdot\boldsymbol{\lambda} - \lambda_j)|  \leq \delta
  \right\}.
\end{align*}
We then set 
\begin{align}\label{eqn:extended_normal_form_solution}
(\widetilde{t}_{j,\mathbf{m}}, n_{j,\mathbf{m}}) = \begin{cases}
  (g_{j,\mathbf{m}}/(\mathbf{m}\cdot\boldsymbol{\lambda}-\lambda_j), 0), & (j,\mathbf{m}) \notin \mathcal{I}_\delta,\\
  (0, g_{j,\mathbf{m}}), & (j,\mathbf{m}) \in \mathcal{I}_\delta.
\end{cases}
\end{align}
With this convention, the coefficients of $\mathbf{t}$ and $\mathbf{n}$ are uniquely determined order by order. The motivation for considering only the imaginary parts of the eigenvalues is to make sure that the normal form depends continuously on the damping (i.e. the real parts of the eigenvalues), reducing to the classical normal form of the limiting conservative system as $\operatorname{Re}\boldsymbol{\lambda}\to\mathbf{0}$, and follows the convention established in \cite{SSMLearnpaper22}.

We will discuss how this extended normal form can be computed from data in \S\ref{sec:data_driven_ssm_reduction_with_equivariance} and how equivariance can be preserved in \S\ref{sec:equivariant_ssm_reduction}.

\subsection{Equivariant SSM reduction}\label{sec:equivariant_ssm_reduction}
In order to preserve equivariance in the reduced model, it is crucial to choose the complement $E^c$ in a way that is compatible with the symmetry group $\mathcal{G}$. In particular, we need to ensure that $E^c$ is also invariant under the action of $\mathcal{G}$, i.e. for all $S \in \mathcal{G}$ we have $SE^c = E^c$. There are two natural ways to achieve this.

\begin{enumerate}[label=(\Roman*)]
  \item \textbf{Spectral complement:} We pick $E^c$ to be the direct sum of the remaining spectral subspaces of $\mathbf{A}$, i.e. $E^c = \bigoplus_{j \notin J} E_j$, where $E = \bigoplus_{j \in J} E_j$. Invariance of $E^c$ under the action of $\mathcal{G}$ is then guaranteed by Theorem~\ref{thm:equivariance-ssm}. However, for large $n$, it is expensive to compute the spectral complement $E^c$ since it requires computing all eigenvalues and eigenvectors of $\mathbf{A}$, which is not feasible for large-scale systems. Thus, in practice, we prefer to use the second option below, which is computationally cheaper and still guarantees equivariance of the reduced model.
  \item \textbf{Orthogonal complement with respect to the equivariant inner product:} For a compact group $\mathcal{G}$, an alternative choice is to define $E^c$ as the orthogonal complement of $E$ with respect to the \textit{equivariant inner product} $\langle \cdot, \cdot \rangle_\mathcal{G}$ as defined in Definition~\ref{def:equivariant_inner_product}.
  Every $S \in \mathcal{G}$ acts as an isometry of $\langle \cdot, \cdot \rangle_\mathcal{G}$, so the orthogonal complement of the $\mathcal{G}$-invariant subspace $E$ is again $\mathcal{G}$-invariant.
\end{enumerate}
  \begin{definition}\label{def:equivariant_inner_product}
    Let $\mathcal{G}$ be a compact linear symmetry group acting on $\mathbb{R}^n$,
    equipped with its normalised Haar measure $\mu$ (so that $\mu(\mathcal{G}) = 1$).
    The \textit{equivariant inner product} $\langle \cdot, \cdot \rangle_\mathcal{G}$
    is defined as
    \begin{align*}
      \langle \mathbf{v}, \mathbf{w} \rangle_{\mathcal{G}}
      := \int_{\mathcal{G}} \langle S\mathbf{v}, S\mathbf{w} \rangle \, \dd\mu(S),
    \end{align*}
    where $\langle \cdot, \cdot \rangle$ is the standard inner product on
    $\mathbb{R}^n$. For a finite group, the Haar measure is the normalised counting
    measure and this reduces to the average
    $\langle \mathbf{v}, \mathbf{w} \rangle_{\mathcal{G}}
    = \frac{1}{|\mathcal{G}|} \sum_{S \in \mathcal{G}} \langle S\mathbf{v}, S\mathbf{w} \rangle$.
  \end{definition}
With any $\mathcal{G}$-invariant choice of $E^c$, it turns out that the reduced dynamics on the SSM $\mathcal{W}(E)$ also inherits an equivariance property from the full system, which can be used to significantly simplify the computations of the parametrisation map $\mathbf{h}$ and the reduced vector field $\mathbf{r}$ as well as to improve the model fidelity of the reduced model. For this we introduce the standard restriction $S|_E:= \mathbf{V}_1^\top S \mathbf{U}_1$ for $S \in \mathcal{G}$ and let:
\begin{align*}
  \mathcal{G}|_E := \{ S|_E, S \in \mathcal{G} \}.
\end{align*}
\begin{lemma}\label{lem:restricted_group_action}
  The restriction $\mathcal{G}|_E$ is a well-defined linear symmetry group acting on $\mathbb{R}^d\cong E$, $\mathcal{G}|_E \subset GL(d,\mathbb{R})$.
\end{lemma}
\begin{proof} Since $E$ is invariant under the action of $\mathcal{G}$ (Theorem~\ref{thm:equivariance-ssm}), and $\mathbf{U}_1\mathbf{V}_1^\top$ is the projection onto $E$ (cf. Figure~\ref{fig:graph_parametrisation}) we have, for any $S_1, S_2 \in \mathcal{G}$,
\begin{align*}
  S_1|_E S_2|_E = \mathbf{V}_1^\top S_1 \mathbf{U}_1 \mathbf{V}_1^\top S_2 \mathbf{U}_1 = \mathbf{V}_1^\top S_1 S_2 \mathbf{U}_1  = (S_1 S_2)|_E,
\end{align*}
where the second equality follows from the fact that the columns of $S_2\mathbf{U}_1$ are in $S_2E=E$, where $\mathbf{U}_1\mathbf{V}_1^\top = \pi_E$ acts as the identity. Thus the restriction is closed under group multiplication and hence a group homomorphism. The result follows.
\end{proof}

We can now show that the parametrisation map $\mathbf{h}$ and the reduced vector field $\mathbf{r}$ are also equivariant with respect to this restricted group action.
\begin{proposition}\label{prop:equivariance_reduced_dynamics} If \eqref{eqn:general_system} is equivariant with respect to a compact linear symmetry group $\mathcal{G}$ and $E$ is a spectral subspace of $\mathbf{A}$, then, in a neighbourhood of the origin $\mathbf{0}$, the parametrisation map $\mathbf{h}: E \rightarrow \mathbb{R}^n$ and the reduced vector field $\mathbf{r}: E \rightarrow E$ are equivariant with respect to the restricted group action $\mathcal{G}|_E$, in the sense that
\begin{align*}
  \mathbf{h}(S|_E \boldsymbol{\eta}) = S \mathbf{h}(\boldsymbol{\eta}), \qquad \mathbf{r}(S|_E \boldsymbol{\eta}) = S|_E \mathbf{r}(\boldsymbol{\eta}),
\end{align*}for all $S \in \mathcal{G}$ and $\boldsymbol{\eta} \in E$.
\end{proposition}
\begin{proof}
Let $S \in \mathcal{G}$ be a linear symmetry of the system, and denote by $\Omega\subset\mathbb{R}^n$ the neighbourhood of $\mathbf{0}$ in which the parametrisation \eqref{eqn:parametrisation_WE} is valid. Note, in the following we will assume that $S\mathbf{y} \in \Omega$ for $\mathbf{y} \in \mathcal{W}(E) \cap \Omega$. This can be achieved by shrinking $\Omega$ to $\bigcap_{S\in\mathcal{G}}S\Omega$, which is non-empty by compactness of $\mathcal{G}$. Then, by virtue of $E$ and $E^c$ being $\mathcal{G}$-invariant, we have for any $\mathbf{y} \in \mathcal{W}(E) \cap \Omega$,
\begin{align*}
  S\vert_E\mathbf{V}_1^\top \mathbf{y} = \mathbf{V}_1^\top S \mathbf{U}_1\mathbf{V}_1^\top \mathbf{y} &= \mathbf{V}_1^\top S \mathbf{U}_1\mathbf{V}_1^\top (\mathbf{U}_1\boldsymbol{\eta} + \mathbf{h}(\boldsymbol{\eta}))\\
  & = \mathbf{V}_1^\top S\mathbf{U}_1\boldsymbol{\eta}= \mathbf{V}_1^\top S(\mathbf{U}_1\boldsymbol{\eta}+\mathbf{h}(\boldsymbol{\eta})) =\mathbf{V}_1^\top S\mathbf{y},
  \end{align*}
where we have used the fact that $S\mathrm{Im}(\mathbf{h}) \subset SE^{c}=E^c\subset\mathrm{ker}(\mathbf{V}_1^\top)$ and wrote $\boldsymbol{\eta} = \mathbf{V}_1^\top \mathbf{y}$ for the last equality. Thus $S|_E \mathbf{V}_1^\top\mathbf{y}$ is the unique reduced coordinate of $S\mathbf{y}$. Lifting this back to $\mathcal{W}(E)$ (using the fact that $\mathcal{W}(E)$ is $S$-invariant by Theorem~\ref{thm:equivariance-ssm}) we thus have
\begin{align*}
  S\mathbf{U}_1\boldsymbol{\eta} + S \mathbf{h}(\boldsymbol{\eta}) = S\mathbf{y}=\mathbf{U}_1S|_E \boldsymbol{\eta}+\mathbf{h}(S|_E \boldsymbol{\eta}).
\end{align*}
Applying the dual bases $\mathbf{V}_1^\top$ and $\mathbf{V}_2^\top$ respectively to this equation (i.e. the projection onto coordinates of $E, E^c$) we have (noting that $\mathbf{V}_1^\top \mathbf{U}_2 = 0$, $\mathbf{V}_2^\top \mathbf{U}_1 = 0$ and that $\mathrm{Im}(\mathbf{U}_1)=E,\mathrm{Im}(\mathbf{h})\subset E^c$ with $E,E^c$ invariant under the action of $S$):
\begin{align*}
  S\mathbf{U}_1\boldsymbol{\eta} &= \pi_E (S\mathbf{U}_1\boldsymbol{\eta} + S \mathbf{h}(\boldsymbol{\eta})) =\mathbf{U}_1 \mathbf{V}_1^\top \left(\mathbf{U}_1S|_E \boldsymbol{\eta} + \mathbf{h}(S|_E \boldsymbol{\eta}) \right) =\mathbf{U}_1 S|_E \boldsymbol{\eta},\\
S \mathbf{h}(\boldsymbol{\eta})&=  \pi_{E^c}\left( S\mathbf{U}_1\boldsymbol{\eta} + S \mathbf{h}(\boldsymbol{\eta}) \right) = \mathbf{U}_2\mathbf{V}_2^\top \left(\mathbf{U}_1S|_E \boldsymbol{\eta} + \mathbf{h}(S|_E \boldsymbol{\eta}) \right) = \mathbf{h}(S|_E \boldsymbol{\eta}).
\end{align*}
Finally, using this, we have
\begin{align*}
  \mathbf{r}(S|_E \boldsymbol{\eta}) &= \mathbf{V}_1^\top \left( \mathbf{A}\mathbf{U}_1S|_E \boldsymbol{\eta} + \mathbf{A}\mathbf{h}(S|_E \boldsymbol{\eta}) + \mathbf{f}(\mathbf{U}_1S|_E \boldsymbol{\eta} + \mathbf{h}(S|_E \boldsymbol{\eta})) \right)\\
  &=\mathbf{V}_1^\top S\left( \mathbf{A}\mathbf{U}_1 \boldsymbol{\eta} + \mathbf{A}\mathbf{h}(\boldsymbol{\eta}) + \mathbf{f}(\mathbf{U}_1 \boldsymbol{\eta} + \mathbf{h}(\boldsymbol{\eta})) \right)\\
  &= S|_E \mathbf{r}(\boldsymbol{\eta}).
\end{align*}
\end{proof}
Having established the properties of the reduced dynamics on the SSM $\mathcal{W}(E)$, we turn to the canonical construction of an equivariant normal form for the reduced dynamics, cf.\ \eqref{eqn:extended_normal_form}. It turns out that the change of variables $\mathbf{f}$ and $\mathbf{n}$, uniquely constructed as per \eqref{eqn:extended_normal_form_solution}, are also equivariant with respect to appropriate restricted group action of $\mathcal{G}$. We will discuss the implications of this equivariance for the data-driven computation of the extended normal form in \S\ref{sec:data_driven_ssm_reduction_with_equivariance}.
\begin{proposition}\label{prop:equivariant_normal_form}
 Suppose \eqref{eqn:general_system} is equivariant with respect to the compact linear symmetry group $\mathcal{G}$. Let $\mathbf{t},\mathbf{n}$ be constructed termwise from \eqref{eqn:extended_normal_form_solution} and, for $S\in\mathcal{G}$, let $\widehat{S}:=\mathbf{W}^{-1}(S|_E)\mathbf{W}$ be the restricted action in the diagonalising coordinates $\mathbf{W}$ defined in \eqref{eqn:def_B_and_Lambda}. Then $\hat{\mathbf{t}}:=\mathbf{W}^{-1}\mathbf{t}$ and $\mathbf{n}$ are both equivariant with respect to the restricted group action of $\mathcal{G}$, i.e.
  \begin{align*}
    \hat{\mathbf{t}}(\widehat{S}\mathbf{z}) = \widehat{S} \hat{\mathbf{t}}(\mathbf{z}), \qquad \mathbf{n}(\widehat{S}\mathbf{z}) = \widehat{S} \mathbf{n}(\mathbf{z}), \qquad \forall S\in\mathcal{G},\ \mathbf{z}\in\mathbb{C}^d.
  \end{align*}
\end{proposition}
\begin{remark}
  Note this equivariance of $\hat{\mathbf{t}}$ is equivalent to $\mathbf{t}(\widehat{S}\mathbf{z}) = S|_E \mathbf{t}(\mathbf{z})$ for all $S\in\mathcal{G},\ \mathbf{z}\in\mathbb{C}^d$.
\end{remark}
\begin{proof} We will exploit uniqueness of the construction of $\mathbf{t},\mathbf{n}$ as described in \S\ref{sec:computing_ssms}. Let $S\in\mathcal{G}$ and define
  \begin{align*}
    \mathbf{t}'(\mathbf{z}):=S|_E^{-1}\mathbf{t}(\widehat{S}\mathbf{z}), \qquad
    \mathbf{n}'(\mathbf{z}):=\widehat{S}^{-1}\mathbf{n}(\widehat{S}\mathbf{z}).
  \end{align*}
 We will show that $(\mathbf{t}',\mathbf{n}')$ satisfies the same conjugacy equation \eqref{eqn:conjugacy} and has the same linear parts and near-resonant support as $(\mathbf{t},\mathbf{n})$. By uniqueness of the solution to the term-wise conditions \eqref{eqn:homological}, we then conclude $\mathbf{t}'=\mathbf{t}$ and $\mathbf{n}'=\mathbf{n}$, which gives the desired equivariance properties.

  Differentiating $\mathbf{t}'$ and using the conjugacy equation \eqref{eqn:conjugacy} at
  $\widehat{S}\mathbf{z}$ together with linearity of the group action and the equivariance of $\mathbf{r}$,
  \begin{align*}
    D\mathbf{t}'(\mathbf{z})\,\mathbf{n}'(\mathbf{z})
    &=(S|_E)^{-1}\,D\mathbf{t}(\widehat{S}\mathbf{z})\widehat{S}\widehat{S}^{-1}\,\mathbf{n}(\widehat{S}\mathbf{z})=(S|_E)^{-1}\,D\mathbf{t}(\widehat{S}\mathbf{z})\,\mathbf{n}(\widehat{S}\mathbf{z})\\
    &=(S|_E)^{-1}\,\mathbf{r}\big(\mathbf{t}(\widehat{S}\mathbf{z})\big)=(S|_E)^{-1}\,\mathbf{r}\big(S|_E\,\mathbf{t}'(\mathbf{z})\big)
    =\mathbf{r}\big(\mathbf{t}'(\mathbf{z})\big),
  \end{align*}
  so $(\mathbf{t}',\mathbf{n}')$ satisfies the same conjugacy equation \eqref{eqn:conjugacy} as $(\mathbf{t},\mathbf{n})$. Next we show that the linear parts of $\mathbf{t}'$ and $\mathbf{n}'$ coincide with those of $\mathbf{t}$ and $\mathbf{n}$. The linear part of $\mathbf{t}'$ is given by
  \begin{align*}
    D\mathbf{t}'(\mathbf{0}) = (S|_E)^{-1}D\mathbf{t}(\mathbf{0})\widehat{S} = (S|_E)^{-1}\mathbf{W}\widehat{S} = (S|_E)^{-1}\mathbf{W}(\mathbf{W}^{-1}(S|_E)\mathbf{W}) = \mathbf{W} = D\mathbf{t}(\mathbf{0}),
  \end{align*}
  and the linear part of $\mathbf{n}'$ is given by
  \begin{align*}
    D\mathbf{n}'(\mathbf{0}) = \widehat{S}^{-1}D\mathbf{n}(\mathbf{0})\widehat{S} = \widehat{S}^{-1}\boldsymbol{\Lambda}\widehat{S} = \boldsymbol{\Lambda} = D\mathbf{n}(\mathbf{0}),
  \end{align*}
  where we have used that $\mathbf{B}=\mathbf{V}_1^\top\mathbf{A}\mathbf{U}_1$ commutes with $S|_E$ for all $S\in\mathcal{G}$, so that $\boldsymbol{\Lambda}$ commutes with $\widehat{S}$. It remains to show that the assignment of nonlinear terms to $\mathbf{t}'$ and $\mathbf{n}'$, i.e. their near-resonant support, is the same as that for $\mathbf{t}$ and $\mathbf{n}$ given in \eqref{eqn:extended_normal_form_solution}. In principle the linear substitution $\mathbf{z}\mapsto \widehat{S}\mathbf{z}$ preserves the degree of a monomial, but a single monomial $\mathbf{z}^\mathbf{m}$ may be mapped to a linear combination of monomials of the same degree under this map.
  
  Let us consider a single monomial $\mathbf{e}_j\mathbf{z}^{\mathbf{m}}$ in the Taylor expansion of $\mathbf{n}$, contributing $z_1^{m_1}\cdots z_d^{m_d}$ to the $j$-th component of $\mathbf{n}$. The assignment of this monomial to $\mathbf{n}$ or $\mathbf{t}$ is determined by the size of $|\mathrm{Im}(\mathbf{m}\cdot\boldsymbol{\lambda}-\lambda_j)|$ relative to the threshold $\delta$. Under the action of $\widehat{S}$, this monomial is mapped to
\begin{align*}
   \widehat{S}^{-1}\mathbf{e}_j(\widehat{S}\mathbf{z})^{\mathbf{m}}
= \Big(\sum_{i=1}^d (\widehat{S}^{-1})_{ij}\,\mathbf{e}_i\Big)
  \prod_{l=1}^d\Big(\sum_{k=1}^d \widehat{S}_{lk}\,z_k\Big)^{m_l},
\end{align*}
which is a linear combination of degree-$|\mathbf{m}|$ monomial vector fields $\mathbf{e}_i\mathbf{z}^{\mathbf{m}'}$
with $|\mathbf{m}'|=|\mathbf{m}|$. Now, $\widehat{S}$ commutes with the diagonal matrix
$\boldsymbol{\Lambda}$, i.e.
$\widehat{S}_{lk}(\lambda_k-\lambda_l)=(\widehat{S}\boldsymbol{\Lambda}-\boldsymbol{\Lambda}\widehat{S})_{lk}=0$,
so $\widehat{S}_{lk}=0$ whenever $\lambda_l\neq\lambda_k$, and likewise for $\widehat{S}^{-1}$. In other
words $\widehat{S}$ couples only coordinates belonging to the same eigenvalue. This implies the following
\begin{enumerate}[label=(\roman*)]
  \item $\widehat{S}^{-1}\mathbf{e}_j$ involves only the basis vectors $\mathbf{e}_i$ with $\lambda_i=\lambda_j$;
  \item each factor $(\widehat{S}\mathbf{z})_l=\sum_{k:\lambda_k=\lambda_l}\widehat{S}_{lk}z_k$ involves only
  variables $z_k$ with $\lambda_k=\lambda_l$, so every monomial $\mathbf{z}^{\mathbf{m}'}$ produced by
  expanding the product carries the same eigenvalue weight,
  $\mathbf{m}'\cdot\boldsymbol{\lambda}=\sum_{l=1}^d m_l\lambda_l=\mathbf{m}\cdot\boldsymbol{\lambda}$.
\end{enumerate}
Hence every term $\mathbf{e}_i\mathbf{z}^{\mathbf{m}'}$ appearing in
$\widehat{S}^{-1}\mathbf{e}_j(\widehat{S}\mathbf{z})^{\mathbf{m}}$ has the same value for the resonance condition as the original term $\mathbf{e}_j\mathbf{z}^{\mathbf{m}}$:
\begin{align*}
  \mathbf{m}'\cdot\boldsymbol{\lambda}-\lambda_i=\mathbf{m}\cdot\boldsymbol{\lambda}-\lambda_j,
\end{align*}
and therefore, for any contribution $\mathbf{e}_i\mathbf{z}^{\mathbf{m}'}$ to $\mathbf{n}'$ or $\mathbf{t}'$, we have $(i,\mathbf{m}')\in\mathcal{I}_\delta$ if and only if $(j,\mathbf{m})\in\mathcal{I}_\delta$. This means that the support of the Taylor expansion of $\mathbf{n}'$ is exactly the same as that of $\mathbf{n}$, and likewise for $\widehat{\mathbf{t}}'$ and $\widehat{\mathbf{t}}$. Since the support of the Taylor expansions of $\mathbf{n}$ and $\mathbf{t}$ are disjoint, this implies the term-wise coefficients which are obtained as solutions of \eqref{eqn:homological} are exactly the same as those of $\mathbf{t}$ and $\mathbf{n}$, so that $\mathbf{t}'=\mathbf{t}$ and $\mathbf{n}'=\mathbf{n}$. The result follows.
\end{proof}

\subsection{Taylor expansions of equivariant functions}\label{sec:taylor_expansion_equivariant_functions} We noted in the previous section that the central ingredient in the construction of equivariant SSM-based reduced order models are functions ($\mathbf{h},\mathbf{r},\mathbf{t},\mathbf{n}$) with integer power Taylor expansions which are equivariant with respect to a linear symmetry group. In this section we examine the characterisation of the permissible terms in such an equivariant expansion as well as the reduction in degrees of freedom that can be achieved by incorporating the symmetry. To begin with, we introduce the following notation:
\begin{definition}
  For a given integer $k\in \mathbb{N}$, we denote by $\boldsymbol{\phi}_k^{(m)}(\mathbf{x})$ the vector of all monomials of degree $k$ in the variables $\mathbf{x} = (x_1,\dots,x_m)$, i.e.
  \begin{align*}
    \boldsymbol{\phi}_k^{(m)}(\mathbf{x}) = (\mathbf{x}^\mathbf{j})_{|\mathbf{j}|=k} = (x_1^k, x_1^{k-1}x_2, \dots, x_m^k)^\top.
  \end{align*}
\end{definition}
\begin{theorem}[{Taylor expansion of equivariant functions}]\label{thm:equivariant_taylor} Suppose $\mathbf{g}:\mathbb{R}^m\rightarrow\mathbb{R}^n$ is $\mathcal{C}^\infty$ and has a convergent Taylor series with integer powers around $\mathbf{0}$,
  \begin{align*}
    \mathbf{g}(\mathbf{x}) = \sum_{k\geq 0}\mathbf{C}_{k}\boldsymbol{\phi}_k^{(m)}(\mathbf{x}),
  \end{align*}
where $\mathbf{C}_k\in\mathbb{R}^{n\times N_k}$ collects the coefficients of the monomials of degree $k$ and $N_k = \binom{k+m-1}{k}$ is the number of such monomials. Suppose further that $\mathbf{g}$ is equivariant with respect to the compact symmetry group $\mathcal{H}\subset\mathrm{GL}(n,\mathbb{R})$, with Haar measure $\mu$ and representation $\sigma$ in $\mathbb{R}^m$, i.e.
\begin{align}\label{eqn:equivariance_condition_theorem}
  \mathbf{g}(\sigma(S)\mathbf{x}) = S\mathbf{g}(\mathbf{x}), \qquad \forall S\in\mathcal{H},\ \mathbf{x}\in\mathbb{R}^m.
\end{align}
Then the Taylor coefficients $\mathbf{C}_{k}$ satisfy the following three properties:
\begin{enumerate}[label=(\roman*)]
  \item Linear constraints on the coefficients. The equivariance condition \eqref{eqn:equivariance_condition_theorem} imposes linear constraints on the Taylor coefficients $\mathbf{C}_{k}$, which can be expressed as
  \begin{align}\label{eqn:linear_constraints_Taylor_series}
    S\mathbf{C}_{k} = \mathbf{C}_{k}\mathbf{D}_{k}(S), \qquad \forall S\in\mathcal{H},\ k\geq 0,
  \end{align}
where $\mathbf{D}_{k}(S)$ is the representation of $\mathcal{H}$ on the space of homogeneous polynomials of degree $k$ in $\mathbb{R}^m$ as defined in \eqref{eqn:representation_on_polynomials}.
  \item The admissible coefficients in the Taylor expansion of $\mathbf{g}$ are given precisely by
  \begin{align*}
  \mathbf{C}_{k}\in \bigcap_{S\in\mathcal{H}} \operatorname{ker}(\mathbf{M}_k[S]),
  \end{align*}
  where $\mathbf{M}_k[S]$ is the linear operator defined by
  \begin{align*}
    \mathbf{M}_k[S](\mathbf{C}) := S\mathbf{C}_{k} - \mathbf{C}_{k}\mathbf{D}_{k}(S).
  \end{align*}
  \item Projection onto the admissible space. Let us define the group-averaging operator $\mathcal{R}_k:\mathbb{R}^{n\times N_k}\rightarrow \mathbb{R}^{n\times N_k}$,
  \begin{align}\label{eqn:def_R_operator}
    \mathcal{R}_k(\mathbf{C}):=\int_{\mathcal{H}}S\mathbf{C}\mathbf{D}_k(S)^{-1}\mathrm{d}\mu(S),
  \end{align}
  which for finite groups $\mathcal{H}$ simplifies to $\mathcal{R}_k=\frac{1}{|\mathcal{H}|}\sum_{S\in\mathcal{H}}S\mathbf{C}\mathbf{D}_k(S)^{-1}$. Then $\mathcal{R}_k$ is a linear projection ($\mathcal{R}_k\circ\mathcal{R}_k = \mathcal{R}_k$), whose image is exactly the set of admissible coefficients
\begin{align*}
  \mathrm{im}\mathcal{R}_k = \bigcap_{S\in\mathcal{H}} \operatorname{ker}(\mathbf{M}_k[S]).
\end{align*}
\end{enumerate}
\end{theorem}
\begin{remark}
Theorem~\ref{thm:equivariant_taylor} can, of course, also be applied with $m=n,\sigma(S)=S$ to the case of equivariant vector fields $\mathbf{g}:\mathbb{R}^n\rightarrow\mathbb{R}^n$.
\end{remark}
\begin{remark}
If we write $\operatorname{vec}(\mathbf{C}_k)$ for the representation of the matrix $\mathbf{C}_k$ as a vector in $\mathbb{R}^{nN_k}$ (columns stacked), then the action of $\mathbf{M}_k[S]$ on $\mathbf{C}_k$ can be expressed as a linear operator on $\operatorname{vec}(\mathbf{C}_k)$, i.e.
\begin{align*}
  \mathbf{M}_k[S]=\mathbf{I}_{N_k}\otimes S - \mathbf{D}_k(S)^\top\otimes \mathbf{I}_n,
\end{align*}
where $\mathbf{I}_{N_k}\otimes S$ is the $nN_k\times nN_k$ block-diagonal matrix with $S$ on each
of the $N_k$ diagonal blocks, and $\mathbf{D}_k(S)^\top\otimes \mathbf{I}_n$ is the $nN_k\times nN_k$
matrix whose $(i,j)$-th $n\times n$ block is $\big(\mathbf{D}_k(S)^\top\big)_{ij}\,\mathbf{I}_n$. Similarly, we can define the stacked operator
  \begin{align}\label{eqn:def_stacked_operator}
    \mathbf{M}_k = \begin{bmatrix}
      \mathbf{I}_{N_k}\otimes S_1 - \mathbf{D}_k(S_1)^\top\otimes \mathbf{I}_n\\
      \vdots\\
      \mathbf{I}_{N_k}\otimes S_{|\mathcal{H}|} - \mathbf{D}_k(S_{|\mathcal{H}|})^\top\otimes \mathbf{I}_n
    \end{bmatrix}.
  \end{align}
  This expression turns out to be very useful for computing the admissible coefficients in practice.
\end{remark}
  \begin{proof}[Proof of Theorem~\ref{thm:equivariant_taylor}] \textit{Let us begin by proving property (i).} By linearity of the action of $\mathcal{H}$, the equivariance condition \eqref{eqn:equivariance_condition_theorem} must hold at any order in the Taylor series meaning we have, for every $k\geq 0$,
\begin{align*}
  \mathbf{C}_k\boldsymbol{\phi}^{(m)}_{k}(\sigma(S)\mathbf{x}) = S \mathbf{C}_k \boldsymbol{\phi}_{k}^{(m)}(\mathbf{x}).
\end{align*}
Let us denote by $\mathbf{D}_k(S)$ the representation of the action of $\mathcal{H}$ on the space of homogeneous polynomials of degree $k$ in $\mathbb{R}^m$, i.e.
\begin{align}\label{eqn:representation_on_polynomials}
  \mathbf{D}_k(S)\boldsymbol{\phi}_{k}^{(m)}(\mathbf{x})= \boldsymbol{\phi}_{k}^{(m)}(\sigma(S)\mathbf{x}).
\end{align}
Then the above equation can be rewritten as
\begin{align*}
  \mathbf{C}_k\mathbf{D}_k(S)\boldsymbol{\phi}_{k}^{(m)}(\mathbf{x}) = S \mathbf{C}_k \boldsymbol{\phi}_{k}^{(m)}(\mathbf{x}),
\end{align*}
which must hold for all $\mathbf{x}\in\mathbb{R}^m$, i.e. we must have 
\begin{align*}
  S \mathbf{C}_k = \mathbf{C}_k\mathbf{D}_k(S), \qquad \forall S\in\mathcal{H},\ k\geq 0,
\end{align*}
which is exactly the statement of property (i).

\textit{Property (ii) then immediately follows, since the admissible coefficients are precisely those that satisfy the linear constraints in property (i).} 

\textit{Finally, we prove property (iii).} Consider the linear map $\rho_k(S):\mathbb{R}^{n\times N_k}\to\mathbb{R}^{n\times N_k}$ on the space of Taylor coefficients of degree $k$, given by
$\rho_k(S)\mathbf{C}:=S\,\mathbf{C}\,\mathbf{D}_k(S)^{-1}$. Since $\mathbf{D}_k$ is a representation of $\mathcal{H}$ we have, for any $S,T\in\mathcal{H}$, $\mathbf{D}_k(S)\mathbf{D}_k(T)=\mathbf{D}_k(ST)$, and thus, for any $\mathbf{C}\in\mathbf{R}^{n\times N_k}$,
\begin{align*}
  \rho_k(S)\rho_k(T)\mathbf{C}=S\,T\,\mathbf{C}\,\mathbf{D}_k(T)^{-1}\mathbf{D}_k(S)^{-1}
  = (ST)\,\mathbf{C}\,\mathbf{D}_k(ST)^{-1}
  = \rho_k(ST)\mathbf{C},
\end{align*}
so $\rho_k$ is a representation of $\mathcal{H}$ on the coefficient space. By property (i) a coefficient $\mathbf{C}_k$ is admissible if and only if $\rho_k(S)\mathbf{C}_k=\mathbf{C}_k$ for all $S\in\mathcal{H}$, i.e. if and only if it is a fixed point for all $\rho_k(S)$, i.e. $\mathbf{C}_k\in\operatorname{Fix}(\rho_k)$. Let us now consider $\mathcal{R}_k(\mathbf{C})$. We have, for any $T\in\mathcal{H}$
\begin{align*}
  \rho_k(T)\mathcal{R}_k(\mathbf{C})&=T\int_{\mathcal{H}}S\mathbf{C}\mathbf{D}_k(S)^{-1}\mathrm{d}\mu(S)\mathbf{D}_k(T)^{-1}\\
  &=\int_{\mathcal{H}}(TS)\mathbf{C}\mathbf{D}_k(TS)^{-1}\mathrm{d}\mu(S)=\mathcal{R}_k(\mathbf{C}),
\end{align*}
where in the final equality we reindexed $S\mapsto ST$ using the group invariance property of the Haar measure. Thus $\mathcal{R}_k(\mathbf{C})$ is fixed by $\rho_k(T)$ for all $T\in\mathcal{H}$, and hence $\operatorname{im}\mathcal{R}_k\subseteq\operatorname{Fix}(\rho_k)$. Conversely, if $\mathbf{C}\in\operatorname{Fix}(\rho_k)$ then $\rho_k(S)\mathbf{C}=\mathbf{C}$ for every $S$ and
\begin{align*}
  \mathcal{R}_k(\mathbf{C})=\int_{\mathcal{H}}\rho_k(S)\mathbf{C}\mathrm{d}\mu(S)=\int_{\mathcal{H}}\mathbf{C}\mathrm{d}\mu(S)=\mathbf{C},
\end{align*}
thus $\mathcal{R}_k$ fixes $\operatorname{Fix}(\rho_k)$ pointwise, i.e.\ $\operatorname{Fix}(\rho_k)\subseteq\operatorname{im}\mathcal{R}_k$ and $\mathcal{R}_k^2=\mathcal{R}_k$. Thus
$\mathcal{R}_k$ is a projection with
$\operatorname{im}\mathcal{R}_k=\operatorname{Fix}(\rho_k)=\bigcap_{S}\operatorname{ker}\mathbf{M}_k[S]$, and
$\mathbf{C}_k$ is admissible iff $\mathbf{C}_k=\mathcal{R}_k(\mathbf{C}_k)$.
\end{proof}
Theorem~\ref{thm:equivariant_taylor} can be used to count the number of independent coefficients in the Taylor expansion of an equivariant function:
\begin{corollary} Let $\mathcal{H}$ be a finite group, then the number of admissible coefficients in the Taylor expansion of an equivariant function at degree $k$,
  \begin{align*}
    p_k := \dim\!\Big(\bigcap_{S\in\mathcal{H}}\operatorname{ker}\mathbf{M}_k[S]\Big),
  \end{align*}
  can be characterised by:
   \begin{align*}
    p_k \;=\; \operatorname{tr}\mathcal{R}_k
        \;=\; \frac{1}{|\mathcal{H}|}\sum_{S\in\mathcal{H}}\operatorname{tr}(S)\,
              \operatorname{tr}\!\big(\mathbf{D}_k(S)^{-1}\big),
  \end{align*}
  where $\mathcal{R}_k$ is the operator defined in \eqref{eqn:def_R_operator}.
\end{corollary}
\begin{proof} We note that $\mathcal{R}_k$ is a
projection onto $\bigcap_{S}\operatorname{ker}\mathbf{M}_k[S]$ by
Theorem~\ref{thm:equivariant_taylor}(iii), so $\operatorname{rank}\mathcal{R}_k=\dim(\operatorname{im}\mathcal{R}_k)=p_k$
and, being idempotent, $\operatorname{rank}\mathcal{R}_k=\operatorname{tr}\mathcal{R}_k$. The final expression for $p_k$ follows from the fact that $\operatorname{tr}(A\otimes B)=\operatorname{tr}(A)\operatorname{tr}(B)$ and that, in vectorised form,
$\operatorname{vec}(\rho_k(S)\mathbf{C})=(\mathbf{D}_k(S)^{-\top}\otimes S)\operatorname{vec}(\mathbf{C})$, for any $S\in\mathcal{H}$.
\end{proof}
\begin{example}[Effect of equivariance on the damped oscillator chain]\label{ex:equivariant_taylor_expansion_damped_oscillator_chain}
We continue the set-up of Example~\ref{ex:chain_of_oscillators} which is equivariant under the symmetry group $\mathcal{G}=\{\mathbf{I},-\mathbf{I}\}$. Letting $E$ be a spectral subspace with associated SSM $\mathcal{W}(E)$, the chart-equivariant parametrisation $\mathbf{h}$ (Prop.~\ref{prop:equivariance_reduced_dynamics}) must satisfy $\mathbf{h}(-\mathbf{z})=-\mathbf{h}(\mathbf{z})$. Moreover, we see that
\begin{align*}
  \mathbf{D}_k(-I) = (-1)^k \mathbf{I}_{N_k},
\end{align*}
thus
\begin{align*}
  \mathcal{R}_k(\mathbf{C}) = \frac{1}{2}\big(\mathbf{C}+(-1)^{k+1}\mathbf{C}\big) = \begin{cases}
    \mathbf{C}, & k\text{ even},\\
    0, & k\text{ odd}.
  \end{cases}
\end{align*}
This means that the equivariant Taylor expansion of $\mathbf{h}$ contains only odd-degree monomials, and thus the number of degrees of freedom in the Taylor expansion is significantly reduced. Note that this ``pruning'' of monomials is a special feature of the sign symmetry $\mathcal{G}=\langle-\mathbf{I}\rangle$; in general, the kernel of $\mathbf{M}_k$ need not be a coordinate subspace, and the admissible coefficients are not obtained by simply discarding certain monomials from the expansion. Instead, we find a reparametrisation $\operatorname{vec}(\mathbf{W}_k)=\mathbf{B}_k\mathbf{w}_k$ to reduce the dimensionality of the fitting problem, as described in further detail in \S\ref{sec:algorithmic_details_essm}.
\end{example}
\section{Data-driven SSM reduction with equivariance}\label{sec:data_driven_ssm_reduction_with_equivariance}
We now show how we can use the above theory to develop a data-driven SSM reduction method that preserves equivariance. The key idea is to incorporate the restrictions obtained in \S\ref{sec:taylor_expansion_equivariant_functions} into the optimisation problem for computing the reduced dynamics from data, which leads to a smaller parameter space and thus a more efficient and robust algorithm.

For the data-driven eSSM reduction we assume that we have access to trajectories sampled uniformly in time, $\mathbf{y}_{k}^{(j)}\in \mathbb{R}^n$, $k=1,\dots,N_j$, $j=1,\dots,J$, generated by the underlying system \eqref{eqn:general_system},
\begin{align}\label{eqn:trajectory_data}
  \mathbf{y}_{k}^{(j)} = \mathbf{x}(t_k;\mathbf{x}_0^{(j)}), \qquad t_k = (k-1)\Delta t, \quad k=1,\dots,N_j,
\end{align}
where $\mathbf{x}(t;\mathbf{x}_0)$ is the solution of \eqref{eqn:general_system} with initial condition $\mathbf{x}_0$ at time $t$. For notational simplicity, we will restrict the presentation of the algorithm to the case of a single trajectory of data, i.e. $J=1$. Information from multiple trajectories can be incorporated into the algorithm by stacking the data matrices:
\begin{align*}
   \mathbf{Y} = [\mathbf{Y}^{(1)},\dots,\mathbf{Y}^{(J)}] \in \mathbb{R}^{n\times \sum_{j=1}^J N_j},
\quad\text{where }\mathbf{Y}^{(j)} = [\mathbf{y}_1^{(j)},\dots,\mathbf{y}_{N_j}^{(j)}] \in \mathbb{R}^{n\times N_j},
\end{align*}
taking care to apply any finite difference approximations of time derivatives per trajectory separately. In addition, we assume that we have knowledge of the \textit{finite} symmetry group $\mathcal{G}$ of the system.

\begin{remark}We explicitly restrict the presentation of the algorithm to finite symmetry groups for clarity of exposition, in principle most steps below can be extended to compact symmetry groups, although additional considerations for the implementation of the algorithm (group quadrature, etc.) would be required.\end{remark}
\subsection{Delay embedding and equivariance}
\label{sec:delay_embedding_and_equivariance}
In applications when observations of the full state $\mathbf{x}\in\mathbb{R}^n$ are not available, it is common practice in SSM reduction and similar data-driven methods to rely on Takens' embedding theorem \cite{takens1981detecting,deyle2011generalized,Sauer91} to reconstruct a suitable state space using delay embedding of lower-dimensional observations \cite{SSMLearnpaper22,FastSSMpaper23}. In particular, for lower-dimensional observations $\mathbf{H}(\mathbf{x})\in\mathbb{R}^m$, $m<n$, a delay embedding is typically constructed as
\begin{align}\label{eqn:delay_embedding}
  \boldsymbol{\Phi}(\mathbf{x})
  = \big(\mathbf{H}(\mathbf{x}),\,\mathbf{H}(\varphi_{\tau}(\mathbf{x})),\,\dots,\,
    \mathbf{H}(\varphi_{(p-1)\tau}(\mathbf{x}))\big) \in \mathbb{R}^{mp},
\end{align}
where $\varphi_t$ denotes the flow map of \eqref{eqn:general_system}, $\tau>0$ is a delay (in practice an integer multiple of the sampling interval $\Delta t$), and $p\in\mathbb{N}$ is the number of delays. By Takens' embedding theorem \cite{takens1981detecting} and its extensions \cite{Sauer91}, for generic observables and $mp \geq 2\dim\mathcal{W}(E)+1$ the map $\boldsymbol{\Phi}$ restricts to an embedding of $\mathcal{W}(E)$, so that the eSSM reduction can be performed on the delay-embedded data.

We note that, if the observable $\mathbf{H}$ is equivariant with respect to the symmetry group $\mathcal{G}$, then the delay embedding $\boldsymbol{\Phi}$ is also equivariant with respect to a lifted representation of $\mathcal{G}$ on the delay-embedded space. This means that the equivariance properties of the original system are preserved under delay embedding, allowing for the application of equivariant SSM reduction techniques even when only lower-dimensional observations are available. This statement is made precise in the following result.

\begin{proposition}\label{prop:delay_embedding_equivariance}
Let \eqref{eqn:general_system} be equivariant with respect to the linear symmetry group $\mathcal{G}\subset GL(n,\mathbb{R})$, and suppose the observable $\mathbf{H}:\mathbb{R}^n\to\mathbb{R}^m$ is $\mathcal{G}$-covariant, i.e. there is a linear representation $\sigma:\mathcal{G}\to GL(m,\mathbb{R})$ such that $\mathbf{H}(S\mathbf{x})=\sigma(S)\mathbf{H}(\mathbf{x})$ for all $S\in\mathcal{G}$, $\mathbf{x}\in\mathbb{R}^n$. Then
\begin{align}\label{eqn:intertwine_identity}
  \boldsymbol{\Phi}(S\mathbf{x})
  = \big(\mathbf{I}_p\otimes\sigma(S)\big)\,\boldsymbol{\Phi}(\mathbf{x}),
  \qquad S\in\mathcal{G},\ \mathbf{x}\in\mathbb{R}^n,
\end{align}
i.e. $\mathcal{G}$ acts on the delay-embedded space $\mathbb{R}^{mp}$ by the block-diagonal representation $\widetilde{\sigma}=\mathbf{I}_p\otimes\sigma$, and the dynamics induced on $\boldsymbol{\Phi}(\mathbb{R}^n)$ are equivariant with respect to $\widetilde{\sigma}(\mathcal{G})$.
\end{proposition} 
\begin{proof}
Since \eqref{eqn:general_system} is $\mathcal{G}$-equivariant, so is the flow map $\varphi_t$ (see proof of
Theorem~\ref{thm:equivariance-ssm}). Thus, each block of \eqref{eqn:delay_embedding}
satisfies
\begin{align*}
  \mathbf{H}\big(\varphi_{k\tau}(S\mathbf{x})\big)
  = \mathbf{H}\big(S\varphi_{k\tau}(\mathbf{x})\big)
  = \sigma(S)\,\mathbf{H}\big(\varphi_{k\tau}(\mathbf{x})\big),
  \qquad k=0,\dots,p-1,
\end{align*}
which completes the proof of \eqref{eqn:intertwine_identity}. In particular, if $\mathbf{x}(t)$ solves \eqref{eqn:general_system}, then so does
$S\mathbf{x}(t)$, and by \eqref{eqn:intertwine_identity} the corresponding embedded
trajectories are $\boldsymbol{\Phi}(\mathbf{x}(t))$ and
$\widetilde{\sigma}(S)\boldsymbol{\Phi}(\mathbf{x}(t))$. Hence $\widetilde{\sigma}(S)$
maps $\boldsymbol{\Phi}(\mathbb{R}^n)$ into itself and sends embedded trajectories to
embedded trajectories, i.e.\ the dynamics induced on $\boldsymbol{\Phi}(\mathbb{R}^n)$
are equivariant with respect to $\widetilde{\sigma}(S)$ for every $S\in\mathcal{G}$.
\end{proof}

\subsection{Algorithmic details of the eSSM reduction method}\label{sec:algorithmic_details_essm} We formulate the equivariant SSM reduction (eSSM) algorithm as an extension of the SSMLearn method introduced in \cite{SSMLearnpaper22}. Our algorithm consists of two main steps: (i) identification and parametrisation of the spectral submanifold $\mathcal{W}(E)$; and (ii) computation of the reduced dynamics on the SSM in extended normal form style. Each of these steps is described in detail below, along with the interaction of the symmetry group $\mathcal{G}$ with each step. Our trajectory data $\mathbf{Y}= [\mathbf{y}_1,\dots,\mathbf{y}_N] \in \mathbb{R}^{n\times N}$ is assumed to be generated from the underlying system \eqref{eqn:general_system} as in \eqref{eqn:trajectory_data}, and we are only able to access the data $\mathbf{Y}$ and the symmetry group $\mathcal{G}$, but not the underlying system \eqref{eqn:general_system} itself.

\paragraph{Step (i): Identification and parametrisation of the spectral submanifold $\mathcal{W}(E)$ from data}
Our ultimate goal in Step~(i) will be to identify equivariant $\mathbf{U}_1, \mathbf{V}_1$ and $\mathbf{h}$ from data such that the following equivariant least-squares objective is minimised:
\begin{align}\label{eqn:equivariant_least_squares}
  \sum_{i=1}^{N} \big\|\, {\mathbf{y}}_i  - {\mathbf{U}}_1\boldsymbol{\eta}_i
  - \mathbf{h}(\boldsymbol{\eta}_i) \,\big\|_{\mathcal{G}}^2,\,\, \text{where}\,\, \boldsymbol{\eta}_i = \mathbf{V}_1^\top\mathbf{y}_i,
\end{align}
together with the graph parametrisation
\begin{align}\label{eqn:graph_parametrisation_data}
  \mathbf{y} = \mathbf{U}_1\boldsymbol{\eta} + \mathbf{h}(\boldsymbol{\eta}),
  \qquad
  \mathbf{h}(\boldsymbol{\eta}) = \sum_{k=2}^M \mathbf{W}_k\,\boldsymbol{\phi}_{k}(\boldsymbol{\eta}), \quad \mathbf{V}_1^\top \mathbf{W}_k = 0, \quad k=2,\dots,M.
\end{align}
To begin with, we note that the equivariant inner product $\langle \cdot, \cdot \rangle_\mathcal{G}$ can be efficiently computed using the Gram matrix
\begin{align*}
  \mathbf{P}_{\mathcal{G}}=\frac{1}{|\mathcal{G}|}\sum_{S\in\mathcal{G}}S^\top S.
\end{align*}
In particular, if we let $\mathbf{P}_{\mathcal{G}} = \mathbf{L}\mathbf{L}^\top$ be the Cholesky factorisation of $\mathbf{P}_{\mathcal{G}}$, then we can compute the equivariant inner product as $\langle \mathbf{v}, \mathbf{w} \rangle_\mathcal{G} = \langle \mathbf{L}^\top \mathbf{v}, \mathbf{L}^\top \mathbf{w} \rangle$. We refer to the coordinates $\widetilde{\mathbf{y}}_i = \mathbf{L}^\top \mathbf{y}_i$ as the whitened coordinates, and we denote the corresponding whitened data matrix as $\widetilde{\mathbf{Y}} = [\widetilde{\mathbf{y}}_1,\dots,\widetilde{\mathbf{y}}_N]$. In these coordinates, the least-squares objective \eqref{eqn:equivariant_least_squares} can be rewritten as
\begin{align}\label{eqn:equivariant_least_squares_whitened}
  \sum_{i=1}^{N} \big\|\, \widetilde{\mathbf{y}}_i  - \widetilde{\mathbf{U}}_1\boldsymbol{\eta}_i
  - \widetilde{\mathbf{h}}(\boldsymbol{\eta}_i) \,\big\|^2,\,\, \text{where}\,\, \boldsymbol{\eta}_i = \widetilde{\mathbf{V}}_1^\top\widetilde{\mathbf{y}}_i,
\end{align}
where $\widetilde{\mathbf{U}}_1 = \mathbf{L}^\top \mathbf{U}_1$ and $\widetilde{\mathbf{V}}_1 = \mathbf{L}^{-1}\mathbf{V}_1$ are the whitened versions of $\mathbf{U}_1$ and $\mathbf{V}_1$, and $\widetilde{\mathbf{h}}(\boldsymbol{\eta}) = \mathbf{L}^\top \mathbf{h}(\boldsymbol{\eta})$ is the whitened version of $\mathbf{h}$. The advantage of this change of coordinates is that we can now work with the standard inner product in the whitened coordinates, which allows us to use standard SVD to solve {linear} $\|\cdot\|_{\mathcal{G}}$-least-squares problems.

\paragraph{Step (i.1): Symmetry-adapted initialisation: fixing the similarity class of $E$}

To begin with, a central property that we want to incorporate in our parametrisation of $\mathcal{W}(E)$ is the $\mathcal{G}$-invariance of $E$ (cf.\ Theorem~\ref{thm:equivariance-ssm}). The identification of such an invariant $E$ necessarily involves a discrete choice which we fix in the symmetry adapted initialisation. In particular, we will show in the following that $d$-dimensional $\mathcal{G}$-invariant subspaces of $\mathbb{R}^n$ can be characterised through similarity classes of the representation of $\mathcal{G}$ restricted to the corresponding subspace. Let us make this more precise: for any $\mathcal{G}$-invariant $d$-dimensional subspace
$F\subseteq\mathbb{R}^n$ with basis matrix $\mathbf{U}\in\mathbb{R}^{n\times d}$ (full
column rank), the columns of $S\mathbf{U}$ lie again in $F$, so that
\begin{align}\label{eqn:restricted_representation1}
  S\mathbf{U}=\mathbf{U}\,\mathbf{R}_{\mathbf{U}}(S),
  \,\,\text{with\,\,}
  \mathbf{R}_{\mathbf{U}}(S)
  =\big(\mathbf{U}^\top\mathbf{U}\big)^{-1}\mathbf{U}^\top S\,\mathbf{U}
  \in GL(d,\mathbb{R}),
  \,\, S\in\mathcal{G},
\end{align}
and injectivity of $\mathbf{U}$ gives
$\mathbf{R}_{\mathbf{U}}(S_1S_2)=\mathbf{R}_{\mathbf{U}}(S_1)\mathbf{R}_{\mathbf{U}}(S_2)$,
i.e.\ $\mathbf{R}_{\mathbf{U}}$ corresponds to the matrix representation of $\mathcal{G}|_F$ in the coordinates $\mathbf{U}$ on $F$. The similarity class $[\mathbf{R}_F]$ of matrices $\mathbf{R}_{\mathbf{U}}$ is then, by definition, invariant under a change of basis of $F$, and thus induces an equivalence relation on $\mathcal{G}$-invariant subspaces.
\begin{definition}[Equivalence of $\mathcal{G}$-invariant subspaces]
  We say $F_1,F_2\in\mathcal{S}_{d,\mathcal{G}}$ are related, $F_1\sim F_2$, if and only if the similarity classes of corresponding matrix representations of $\mathcal{G}$ are equal, $[\mathbf{R}_{F_1}]=[\mathbf{R}_{F_2}]$.
\end{definition}
Writing
\begin{align*}
  \mathcal{S}_{d,\mathcal{G}}:=\{E\subset \mathbb{R}^n,\ \mathrm{dim}E=d,\ \text{$E$ is $\mathcal{G}$-invariant}\},
\end{align*}
it is then straightforward to show that this relation induces an equivalence relation on the set $\mathcal{S}_{d,\mathcal{G}}$, and thus $\mathcal{S}_{d,\mathcal{G}}$ can be written as a disjoint union of such equivalence classes, henceforth referred to as \textit{similarity classes of invariant subspaces}. In addition, we have the following result.
\begin{proposition}[Classes of invariant subspaces]\label{prop:invariant_subspace_classes}
Let $\mathcal{G}\subset GL(n,\mathbb{R})$ be a finite linear group and $1\le d\le n$.
\begin{enumerate}[label=(\roman*)]
  \item \emph{Finiteness.} Up to similarity, $\mathcal{G}$ admits only finitely many
  $d$-dimensional matrix representations. In particular, $F\mapsto[\mathbf{R}_F]$
  partitions $\mathcal{S}_{d,\mathcal{G}}$ into finitely many similarity classes $[\mathcal{E}_1],\dots,[\mathcal{E}_L]$ of invariant subspaces.
  \item \emph{Rigidity.} If $[0,1]\ni t\mapsto\mathbf{U}(t)\in\mathbb{R}^{n\times d}$
  is continuous with full column rank and every $F(t):=\operatorname{span}\mathbf{U}(t)$
  is $\mathcal{G}$-invariant, then $\mathbf{R}_{\mathbf{U}(t)}\sim\mathbf{R}_{\mathbf{U}(0)}$
  for all $t$, i.e. a continuous path of invariant subspaces never leaves its class, $[F(t)]=[F(0)]$.
\end{enumerate}
\end{proposition}
\begin{proof}
Since every $S\in\mathcal{G}$ satisfies $S^{|\mathcal{G}|}=\mathbf{I}$ (Lagrange's theorem), we have
\begin{align*}
  \mathbf{R}(S)^{|\mathcal{G}|} = \mathbf{R}(S^{|\mathcal{G}|}) = \mathbf{R}(\mathbf{I}) = \mathbf{I},
\end{align*}
thus all eigenvalues of $\mathbf{R}(S)$ must be $|\mathcal{G}|$-th roots of unity. We now resort to the use of character theory of linear representations of groups as per \cite[\S 2]{serre1977linear}: firstly we write $\chi_{\mathbf{R}}(S):=\operatorname{tr}\mathbf{R}(S)$ for the character of a representation $\mathbf{R}$. Secondly, we observe that $\chi_{\mathbf{R}}$ maps $\mathcal{G}$ into the finite set $\Sigma_d$ of sums of $d$ such roots of unity, so only finitely many characters occur. Now we note that Corollary 2 of \cite[\S2.3]{serre1977linear} states that representations with the same character are similar over $\mathbb{C}$, i.e.\ if $\chi_{\mathbf{R}}=\chi_{\mathbf{R}'}$, then there is an invertible $\mathbf{X}\in GL(d,\mathbb{C})$ such that $\mathbf{R}'\mathbf{X}=\mathbf{X}\mathbf{R}$. Writing $\mathbf{P}=\Re\mathbf{X},\mathbf{Q}=\Im\mathbf{X}$ for the real and imaginary part respectively, we have
\begin{align*}
  \mathbf{R}'\mathbf{P}=\mathbf{P}\mathbf{R},\quad \mathbf{R}'\mathbf{Q}=\mathbf{Q}\mathbf{R},
\end{align*}
and since $f:t\mapsto \det (P+t\mathbf{Q})$ is a non-zero polynomial on $\mathbb{C}$ ($f(i)=\operatorname{det}(\mathbf{X})\neq 0$), there is a real $t^*$ such that $f(t^*)\neq 0$, i.e. for which $\mathbf{P}+t^*\mathbf{Q}$ is invertible. Therefore $\mathbf{R}'$ and $\mathbf{R}$ are similar over $\mathbb{R}$ and so $[\mathbf{R}]=[\mathbf{R}']$. This completes the proof of (i).

For (ii), we note that the formula in
\eqref{eqn:restricted_representation1} shows that
$\chi_t(S):=\operatorname{tr}\mathbf{R}_{\mathbf{U}(t)}(S)$ is continuous in $t$. Moreover, $\mathbf{R}_{\mathbf{U}(t)}(S)$ represents the restriction of $S$ to $F(t)$, so its spectrum consists of $d$ eigenvalues of $S$, counted with multiplicity, and thus
$\chi_t(S)$ can only take finitely many values. Any continuous function taking a finite number of values on $[0,1]$ is constant, so $\chi_t=\chi_0$ and the similarity $\mathbf{R}_{\mathbf{U}(t)}\sim\mathbf{R}_{\mathbf{U}(0)}$ follows as in the proof of (i).
\end{proof}

The immediate consequence of this result is that once we fix the invariant subspace similarity class $[E]$ of $E$, no continuous optimisation algorithm on $\mathcal{S}_{d,\mathcal{G}}$ can move the equivariant subspace away from $[E]$. Thus it makes sense to fix $[E]$ once (at initialisation) through the representation $\mathbf{R}$ of $\mathcal{G}$ on $E$ before using continuous optimisation in Step~(i.3) to move inside $[E]$ jointly with the graph parametrisation of $\mathcal{W}(E)$. Conveniently it is this rigidity which allows us to fix the equivariant basis in Step~(i.2) thus resulting in a tangible and efficient algorithm. This initialisation is precisely the purpose of Step~(i.1): we identify from data the correct class $[E_0]$ (equivalently, the restricted representation $\mathbf{R}_0$ of $\mathcal{G}$) together with an initial subspace $E_0$ within it.

To enforce $\mathcal{G}$-invariance of $E_0$ given a finite data sample, we symmetrise the data by orbit augmentation. In particular, representing the action of an element $S \in \mathcal{G}$ in the whitened coordinates as $\widetilde{S} := \mathbf{L}^\top S \mathbf{L}^{-\top} \in O(n,\mathbb{R})$ we can, for finite symmetry groups $\mathcal{G}$ define the orbit-augmented snapshot matrix
\begin{align}\label{eqn:orbit_augmented_snapshot_matrix_Y_G}
  \widetilde{\mathbf{Y}}_{\mathcal{G}} := \big[\,\widetilde{S}\widetilde{\mathbf{Y}}\,\big]_{S \in \mathcal{G}}
  = \big[\,\widetilde{S}_1\widetilde{\mathbf{Y}} \mid \cdots \mid \widetilde{S}_{|\mathcal{G}|}\widetilde{\mathbf{Y}}\,\big]
  \in \mathbb{R}^{n \times |\mathcal{G}| N}.
\end{align}
\begin{lemma}\label{lem:invariance_of_eigenspaces_orbit_augmented}
  Any eigenspace of $\widetilde{\mathbf{Y}}_{\mathcal{G}}\widetilde{\mathbf{Y}}_{\mathcal{G}}^\top$ is invariant under the whitened action of the symmetry group $\mathcal{G}$. In particular, the leading left singular vectors of $\widetilde{\mathbf{Y}}_{\mathcal{G}}$ span a $\mathcal{G}$-invariant subspace provided $\sigma_d>\sigma_{d+1}$, where $\sigma_d$ and $\sigma_{d+1}$ are the $d$-th and $(d+1)$-th singular values of $\widetilde{\mathbf{Y}}_{\mathcal{G}}$, respectively.
\end{lemma}
\begin{proof}
  It suffices to show that, for any $S_0\in\mathcal{G}$, the left Gram matrix $\widetilde{\mathbf{Y}}_{\mathcal{G}}\widetilde{\mathbf{Y}}_{\mathcal{G}}^\top$ commutes with $\widetilde{S}_0$. We have
  \begin{align*}
    \widetilde{S}_0 \widetilde{\mathbf{Y}}_{\mathcal{G}}\widetilde{\mathbf{Y}}_{\mathcal{G}}^\top &= \sum_{S\in\mathcal{G}} \widetilde{S}_0 \widetilde{S}\,\widetilde{\mathbf{Y}}\widetilde{\mathbf{Y}}^\top\,\widetilde{S}^\top=\sum_{S_1=S_0S\in\mathcal{G}} \widetilde{S}_1\,\widetilde{\mathbf{Y}}\widetilde{\mathbf{Y}}^\top\,\widetilde{S}_1^\top\widetilde{S}_0 = \widetilde{\mathbf{Y}}_{\mathcal{G}}\widetilde{\mathbf{Y}}_{\mathcal{G}}^\top \widetilde{S}_0,
  \end{align*}
  where in the second equality we changed the dummy variable of the sum to $S_1=S_0S$ which is still in $\mathcal{G}$ since $\mathcal{G}$ is a group, and used orthogonality of $\widetilde{S}_0$.
\end{proof}

 The leading left singular vectors of $\widetilde{\mathbf{Y}}_{\mathcal{G}}$ then span an exactly $\mathcal{G}$-invariant subspace, which we take as the initial tangent space of the SSM. To obtain a $d$-dimensional reduced order model we thus compute the $d$-dimensional truncated SVD on the augmented snapshot matrix $\widetilde{\mathbf{Y}}_{\mathcal{G}} $,
\begin{align*}
  \widetilde{\mathbf{Y}}_{\mathcal{G}}  \approx \widetilde{\mathbf{U}}\widetilde{\boldsymbol{\Sigma}}\widetilde{\mathbf{V}}^\top,
\end{align*}
where $\widetilde{\mathbf{U}} \in \mathbb{R}^{n\times d}$, $\widetilde{\boldsymbol{\Sigma}} \in \mathbb{R}^{d\times d}$, and $\widetilde{\mathbf{V}} \in \mathbb{R}^{|\mathcal{G}|N\times d}$.

\begin{remark}
  In practice, we do not need to form the orbit-augmented snapshot matrix $\widetilde{\mathbf{Y}}_{\mathcal{G}}$ explicitly. Since $\widetilde{\mathbf{Y}}_{\mathcal{G}}\widetilde{\mathbf{Y}}_{\mathcal{G}}^{\top} = \sum_{S\in\mathcal{G}} \tilde{S}\,\widetilde{\mathbf{Y}}\widetilde{\mathbf{Y}}^{\top}\tilde{S}^{\top}$, its truncated SVD reduces to a symmetric eigenproblem, which we solve for large state dimensions by randomised subspace iteration \cite{halko2011finding}, requiring only matrix-vector products with $\widetilde{\mathbf{Y}}$, its transpose, and the representation matrices.
\end{remark}

Writing $\widetilde{\mathbf{U}}_0:=\widetilde{\mathbf{U}}$ for the orthonormal leading left singular
vectors of $\widetilde{\mathbf{Y}}_{\mathcal{G}}$, the initial primal and dual bases of $E$
(cf.\ \S\ref{sec:equivariant_ssm_reduction}) are then chosen as
\begin{align*}
  \mathbf{U}_{1,0} = \mathbf{L}^{-\top}\widetilde{\mathbf{U}}_0, \qquad
  \mathbf{V}_{1,0} = \mathbf{L}\,\widetilde{\mathbf{U}}_0 .
\end{align*}
The representation $\mathbf{R}_0$ of $\mathcal{G}$ on $E$ is then fixed by $\mathbf{R}_0(S)=\mathbf{V}_{1,0}^{\top}S\mathbf{U}_{1,0}$.

\begin{remark}
Since $\widetilde{\mathbf{U}}_0$ has orthonormal columns, the reduced coordinate
$\boldsymbol{\eta}=\mathbf{V}_{1,0}^\top\mathbf{y}=\widetilde{\mathbf{U}}_0^\top\widetilde{\mathbf{y}}$
inherits the scale of the data along $E$: its entries are of the order of the singular values
$\widetilde{\boldsymbol{\Sigma}}$, which are typically small and of disparate magnitude across the
retained modes, so the monomials $\boldsymbol{\eta}^{\mathbf{m}}$, $2\le|\mathbf{m}|\le M$, span many
orders of magnitude and the least-squares system \eqref{eqn:equivariant_least_squares} becomes severely ill-conditioned. As in
\cite{SSMLearnpaper22,FastSSMpaper23}, we remedy this by rescaling the reduced coordinate, in a way that leaves the reduced model unchanged. We use the diagonal rescaling
\begin{align}\label{eqn:reduced_coordinate_rescaling}
  \mathbf{D}=\frac{1}{\sqrt{|G|N}}\widetilde{\boldsymbol{\Sigma}},
\end{align}
and expand the graph in the rescaled coordinate $\mathbf{D}^{-1}\boldsymbol{\eta}$, i.e.\
$\widetilde{\mathbf{h}}(\boldsymbol{\eta})=\sum_{k=2}^{M}\widetilde{\mathbf{W}}_k\,\boldsymbol{\phi}_k(\mathbf{D}^{-1}\boldsymbol{\eta})$.
This scaling is compatible with the equivariant structure of the problem since $\widetilde{\boldsymbol{\Sigma}}$ commutes with the representation $\mathbf{R}_0$ of $\mathcal{G}$ on $E$. Indeed, combining Lemma~\ref{lem:invariance_of_eigenspaces_orbit_augmented} with the $\mathcal{G}$-invariance of $E$, i.e.\
$\widetilde{S}\widetilde{\mathbf{U}}_0=\widetilde{\mathbf{U}}_0\mathbf{R}_0(S)$, and the exact
eigen-relation
$\widetilde{\mathbf{Y}}_{\mathcal{G}}\widetilde{\mathbf{Y}}_{\mathcal{G}}^\top\widetilde{\mathbf{U}}_0
=\widetilde{\mathbf{U}}_0\widetilde{\boldsymbol{\Sigma}}^2$, we find that
$\widetilde{\boldsymbol{\Sigma}}^2$ commutes with $\mathbf{R}_0$, thus so does $\mathbf{D}$. The rescaling \eqref{eqn:reduced_coordinate_rescaling} is similar to the rescaling used in \cite{SSMLearnpaper22,FastSSMpaper23}, but is more convenient for the equivariant setting since it is expressed in terms of the whitened coordinates.
\end{remark}

\paragraph{Step (i.2): Equivariant bases for the chart $\widetilde{\mathbf{U}}_1$ and the graph
coefficients $\widetilde{\mathbf{W}}_k$ of $\widetilde{\mathbf{h}}$}
Step~(i.1) fixes the reduced representation
$\mathbf{R}_0(S)=\widetilde{\mathbf{U}}_0^\top\widetilde{S}\,\widetilde{\mathbf{U}}_0$
of $\mathcal{G}$ on $E$, in this step we will construct the equivariant bases for the chart $\widetilde{\mathbf{U}}_1$ and the graph coefficients $\widetilde{\mathbf{W}}_k$. To keep the representation of $\mathcal{G}$ unchanged as we refine the
chart, $\widetilde{\mathbf{U}}_1$ must satisfy $\widetilde{S}\,\widetilde{\mathbf{U}}_1
=\widetilde{\mathbf{U}}_1\mathbf{R}_0(S)$ for all $S\in\mathcal{G}$, which is precisely the
degree-one instance of the equivariance condition of Theorem~\ref{thm:equivariant_taylor}
(the case $k=1$, $\mathbf{D}_1=\mathbf{R}_0$). The chart and the graph coefficients are therefore
the admissible blocks of equivariant monomial coefficients: letting $\widetilde{\mathbf{B}}_k$ denote a
whitened orthonormal basis of the nullspace of the stacked operator $\mathbf{M}_k$
\eqref{eqn:def_stacked_operator} built from $\mathbf{R}_0$, we write
\begin{align}\label{eqn:chart_and_graph_bases}
  \operatorname{vec}(\widetilde{\mathbf{U}}_1)=\widetilde{\mathbf{B}}_1\,\mathbf{c},
  \quad \mathbf{c}\in\mathbb{R}^{p_1},
  \qquad
  \operatorname{vec}(\widetilde{\mathbf{W}}_k)=\widetilde{\mathbf{B}}_k\,\mathbf{w}_k,
  \quad \mathbf{w}_k\in\mathbb{R}^{p_k},\ \ k=2,\dots,M.
\end{align}
In this parametrisation we have, for any $S\in\mathcal{G}$,
\begin{align*}
  \widetilde{S}\widetilde{\mathbf{U}}_1=\widetilde{\mathbf{U}}_1\mathbf{R}_0(S),
\end{align*}
by construction, thus the reduced representation of $\mathcal{G}$ on $E$ is preserved. This ensures also that the parametrisations $\operatorname{unvec}(\widetilde{\mathbf{B}}_k\mathbf{w}_k), 2\leq k\leq M,$ satisfy the linear constraints \eqref{eqn:linear_constraints_Taylor_series} for any choice of $\mathbf{c}$ and, hence, the bases $\widetilde{\mathbf{B}}_k$ are independent of $\mathbf{c}$ and never need recomputing during the optimisation of Step~(i.3). Finally we note that the rescaling $\mathbf{D}$ introduced in \eqref{eqn:reduced_coordinate_rescaling} commutes with $\mathbf{R}_0$, meaning that the $\mathcal{G}$-equivariant Taylor series expansion for the rescaled coordinate
$\mathbf{D}^{-1}\boldsymbol{\eta}$ carries the same representation $\mathbf{R}_0$, and the bases
$\widetilde{\mathbf{B}}_k$ are unchanged by the rescaling.

\paragraph{Step (i.3): The reduced constrained least-squares problem}
Substituting \eqref{eqn:chart_and_graph_bases} into the whitened objective
\eqref{eqn:equivariant_least_squares_whitened} reduces the fit to a nonlinear constrained least-squares problem in the finite
coefficient vectors $\mathbf{c}$ and $\mathbf{w}=(\mathbf{w}_2,\dots,\mathbf{w}_M)$,
\begin{align}\label{eqn:reduced_constrained_ls}
  \min_{\mathbf{c},\,\mathbf{w}}\ \sum_{i=1}^{N}\Big\|\,\widetilde{\mathbf{y}}_i
  -\widetilde{\mathbf{U}}_1\boldsymbol{\eta}_i
  -\sum_{k=2}^{M}\widetilde{\mathbf{W}}_k\,\boldsymbol{\phi}_k(\mathbf{D}^{-1}\boldsymbol{\eta}_i)\,\Big\|^2,
  \qquad \boldsymbol{\eta}_i=\widetilde{\mathbf{U}}_1^\top\widetilde{\mathbf{y}}_i,
\end{align}
where $\widetilde{\mathbf{U}}_1=\operatorname{unvec}(\widetilde{\mathbf{B}}_1\mathbf{c})$, $\widetilde{\mathbf{W}}_k=\operatorname{unvec}(\widetilde{\mathbf{B}}_k\mathbf{w}_k)$ are the matrix forms of
$\widetilde{\mathbf{B}}_1\mathbf{c}$, $\widetilde{\mathbf{B}}_k\mathbf{w}_k$, subject to
\begin{align}\label{eqn:reduced_constraints}
  \text{(a)}\ \ \widetilde{\mathbf{U}}_1^\top\widetilde{\mathbf{U}}_1=\mathbf{I}_d,
  \qquad
  \text{(b)}\ \ \widetilde{\mathbf{U}}_1^\top\widetilde{\mathbf{W}}_k=\mathbf{0},\ \ k=2,\dots,M,
\end{align}
where (b) is
the graph condition, placing the nonlinear part off the tangent space so that
$\boldsymbol{\eta}_i=\widetilde{\mathbf{U}}_1^\top\widetilde{\mathbf{y}}_i$ is consistent. We solve this using a standard quasi-Newton constrained solver, warm-started at the $\mathbf{c}$ representing $\widetilde{\mathbf{U}}_0$ and with $\mathbf{w}_k=\mathbf{0}$.

\paragraph{Step (ii): Computation of the reduced dynamics on the SSM in extended normal form style} 
We first estimate the
linear part $\mathbf{B}=D\mathbf{r}(\mathbf{0})$ of the reduced dynamics $\dot{\boldsymbol{\eta}}=\mathbf{r}(\boldsymbol{\eta})$
by regression from the projected snapshot data $\boldsymbol{\Xi},\boldsymbol{\Xi}'\in\mathbb{R}^{d\times (N-1)}$, given by
\begin{align*}
  \boldsymbol{\Xi} = [\boldsymbol{\eta}_1,\dots,\boldsymbol{\eta}_{N-1}],\quad
  \boldsymbol{\Xi}' = [\boldsymbol{\eta}_2,\dots,\boldsymbol{\eta}_{N}],
\end{align*}
where $\boldsymbol{\eta}_i = \mathbf{V}_1^\top\mathbf{y}_i$ are the projected data points in the reduced coordinates. Since the true (unknown) value of $\mathbf{B}$ is $\mathcal{G}|_{E}$-equivariant, we proceed similarly to Step~(i) by forming the orbit-augmented projected snapshot matrices
\begin{align*}
  \boldsymbol{\Xi}_{\mathcal{G}} = [\,S|_E\boldsymbol{\Xi}\,]_{S\in\mathcal{G}} \in \mathbb{R}^{d\times |\mathcal{G}| (N-1)}, \qquad
  \boldsymbol{\Xi}'_{\mathcal{G}} = [\,S|_E\boldsymbol{\Xi}'\,]_{S\in\mathcal{G}} \in \mathbb{R}^{d\times |\mathcal{G}| (N-1)},
\end{align*}
analogously to \eqref{eqn:orbit_augmented_snapshot_matrix_Y_G}, where $S|_E=\mathbf{V}_1^\top S\mathbf{U}_1$, and then compute the least-squares estimate of $\mathbf{B}_{\Delta t}=D\boldsymbol{\phi}_{\mathbf{r}}^{\Delta t}(\mathbf{0})=e^{\mathbf{B}\Delta t}$, the linearised one-step map, from the orbit-augmented data as
\begin{align*}
  \mathbf{B}_{\Delta t} = \operatorname*{argmin}_{\mathbf{B}_{\Delta t}}\sum_{S\in\mathcal{G}}\sum_{i=1}^{N-1}\big\|\,S|_E\boldsymbol{\eta}_{i+1}-\mathbf{B}_{\Delta t}S|_E\boldsymbol{\eta}_i\,\big\|_2^2,
  \qquad
  \boldsymbol{\eta}_i = \mathbf{V}_1^\top\mathbf{y}_i,
\end{align*}
i.e.\ $\mathbf{B}_{\Delta t}
= \boldsymbol{\Xi}'_{\mathcal{G}}\boldsymbol{\Xi}_{\mathcal{G}}^{\dagger}
= \boldsymbol{\Xi}'_{\mathcal{G}}\boldsymbol{\Xi}_{\mathcal{G}}^{\top}
  \big(\boldsymbol{\Xi}_{\mathcal{G}}\boldsymbol{\Xi}_{\mathcal{G}}^{\top}\big)^{-1}$
is the dynamic mode decomposition in the reduced coordinates $\boldsymbol{\eta}$.

Since $\mathbf{B}_{\Delta t}$
is the linearised time-$\Delta t$ map, we recover the continuous-time generator $\mathbf{B}=D\mathbf{r}(\mathbf{0})$ from \eqref{eqn:def_B_and_Lambda} by $\mathbf{B}=\tfrac{1}{\Delta t}\log\mathbf{B}_{\Delta t}$.
In practice we compute the eigendecomposition $\mathbf{B}_{\Delta t}=\mathbf{W}\operatorname{diag}(\mu_1,\dots,\mu_d)\mathbf{W}^{-1}$
and set
\begin{align*}
  \mathbf{B}=\mathbf{W}\boldsymbol{\Lambda}\mathbf{W}^{-1},\qquad
  \boldsymbol{\Lambda}=\tfrac{1}{\Delta t}\operatorname{diag}(\log\mu_1,\dots,\log\mu_d),
\end{align*}
so that the eigenvectors $\mathbf{W}$ and continuous eigenvalues $\boldsymbol{\lambda}$ form the basis of the
extended normal-form parametrisation and enter the resonance condition $\mathbf{m}\cdot\boldsymbol{\lambda}-\lambda_j$
directly. We will proceed to determine ${\mathbf{t}}$ and $\mathbf{n}$ using truncated Taylor expansions up to order $M_{\text{ROD}}$, i.e.
\begin{align}\label{eqn:expansion_for_n}
  \mathbf{n}(\mathbf{z};\mathbf{N}) &= \boldsymbol{\Lambda}\mathbf{z} + \sum_{k=2}^{M_{\text{ROD}}}\mathbf{N}_k\,\boldsymbol{\phi}_{k}(\mathbf{z}),\\\label{eqn:expansion_for_t}
  \mathbf{t}(\mathbf{z};\mathbf{T}) &= \mathbf{W}\Big(\mathbf{z} + \sum_{k=2}^{M_{\text{ROD}}}\mathbf{T}_k\,\boldsymbol{\phi}_{k}(\mathbf{z})\Big),\quad
  \mathbf{t}^{-1}(\boldsymbol{\eta};\mathbf{T}^{\star}) = \mathbf{W}^{-1}\boldsymbol{\eta} + \sum_{k=2}^{M_{\text{ROD}}}\mathbf{T}^{\star}_k\,\boldsymbol{\phi}_{k}(\mathbf{W}^{-1}\boldsymbol{\eta}).
\end{align}
To determine the coefficients $\mathbf{N}_k$ and $\mathbf{T}_k^{\star}$ we proceed in three steps: (a) support selection, (b) equivariant reparametrisation, and (c) coefficient fit.

\textit{(a) Support selection.} In the above expansions, the support of the coefficients $\mathbf{N}_k, \mathbf{T}_k$ and $\mathbf{T}_k^\star$ is determined by the near resonance condition \eqref{eqn:extended_normal_form_solution}, i.e. we set
\begin{align*}
  \operatorname{supp}(\mathbf{N}_k) = \mathcal{I}_\delta,\qquad
  \operatorname{supp}(\mathbf{T}_k) = \operatorname{supp}(\mathbf{T}_k^{\star})= \mathcal{I}_\delta^{\,c},\qquad k=2,\dots,M,
\end{align*}
where $\mathcal{I}_\delta = \{(j,\mathbf{m}) : |\operatorname{Im}(\mathbf{m}\cdot\boldsymbol{\lambda}-\lambda_j)| < \delta\}$.

\textit{(b) Equivariant reparametrisation.} Proposition~\ref{prop:equivariant_normal_form} shows that, if the support of a Taylor series is chosen according to the near-resonance condition \eqref{eqn:extended_normal_form_solution}, then the support of the coefficients is invariant under the action of the symmetry group $\mathcal{G}$. This means, that the support restriction and the equivariance constraint are compatible and may be imposed simultaneously. Concretely, fix a degree $k$ and split the degree-$k$ coefficient matrices according to the partition of the
index set $\{1,\dots,d\}\times\{|\mathbf{m}|=k\}$ into resonant and non-resonant entries,
\begin{align*}
  \mathbb{C}^{d\times N_k}=V_k^{\mathbf{n}}\oplus V_k^{\mathbf{t}},
\end{align*}
where $V_k^{\mathbf{n}}$ ($V_k^{\mathbf{t}}$) consists of those $\mathbf{C}$ supported on the resonant
indices $\mathcal{I}_\delta$ (on their complement). 
 By the proof of
Proposition~\ref{prop:equivariant_normal_form} the diagonalised action $\widehat{S}$ commutes with
$\boldsymbol{\Lambda}$, hence preserves the resonance value $\mathbf{m}\cdot\boldsymbol{\lambda}-\lambda_j$ of every
monomial; the equivariance operator
\begin{align*}
  \mathbf{M}_k[\widehat{S}](\mathbf{C})=\widehat{S}\mathbf{C}-\mathbf{C}\,\mathbf{D}_k(\widehat{S})
\end{align*}
therefore maps each of $V_k^{\mathbf{n}}$ and $V_k^{\mathbf{t}}$ into itself. We may consequently restrict
$\mathbf{M}_k[\widehat{S}]$ to each support subspace and impose equivariance there, so that the admissible
coefficients of $\mathbf{n}$ and $\mathbf{t}^{-1}$ are
\begin{align*}
  \mathbf{N}_k\in\bigcap_{S\in\mathcal{G}}\ker\!\big(\mathbf{M}_k[\widehat{S}]\big|_{V_k^{\mathbf{n}}}\big),
  \qquad
  \mathbf{T}^{\star}_k\in\bigcap_{S\in\mathcal{G}}\ker\!\big(\mathbf{M}_k[\widehat{S}]\big|_{V_k^{\mathbf{t}}}\big).
\end{align*}
If we write $\mathbf{M}_k[\widehat{S}]$ in vectorised form as $\mathbf{I}_{N_k}\otimes\widehat{S}-\mathbf{D}_k(\widehat{S})^{\!\top}\otimes\mathbf{I}_d$,
then the restriction to the support subspaces is simply the deletion of the rows and columns indexed outside the relevant support, and the admissible coefficients are the nullspace of the resulting smaller matrix, computed by SVD analogously to Step~(ii). Writing $\mathbf{B}_k^{\mathbf{n}}$ and $\mathbf{B}_k^{\mathbf{t}}$ for bases of these nullspaces we have, similarly to \eqref{eqn:chart_and_graph_bases},
\begin{align}\label{eqn:equivariant_basis_N_and_T}
  \operatorname{vec}(\mathbf{N}_k)=\mathbf{B}_k^{\mathbf{n}}\,\mathbf{c}_k^{\mathbf{n}},\qquad
  \operatorname{vec}(\mathbf{T}_k)=\mathbf{B}_k^{\mathbf{t}}\,\mathbf{c}_k^{\mathbf{t}},\qquad
  \operatorname{vec}(\mathbf{T}^{\star}_k)=\mathbf{B}_k^{\mathbf{t}}\,\mathbf{c}_k^{\mathbf{t},\star},
\end{align}
with free coefficient vectors $\mathbf{c}_k^{\mathbf{n}},\mathbf{c}_k^{\mathbf{t}}$ counting the equivariant
resonant, respectively non-resonant, monomials at degree $k$.

\textit{(c) Coefficient fit.} The final step is to use this reduced parametrisation together with the trajectory data to infer the free coefficients $\mathbf{c}_k^{\mathbf{n}},\mathbf{c}_k^{\mathbf{t},\star}$ of the reduced dynamics and the normal form transformation. We do this by firstly plugging the equivariant parametrisation \eqref{eqn:equivariant_basis_N_and_T} into the expansions \eqref{eqn:expansion_for_n}-\eqref{eqn:expansion_for_t}, and then minimising the error in the conjugacy equation \eqref{eqn:conjugacy} over the data:
\begin{align*}
  (\mathbf{c}^{\mathbf{n}},\mathbf{c}^{\mathbf{t},\star})
  =\operatorname*{argmin}_{\mathbf{c}^{\mathbf{n}},\mathbf{c}^{\mathbf{t},\star}}
  \sum_{i}\Big\|
  D\mathbf{t}^{-1}(\boldsymbol{\eta}_i)\,\dot{\boldsymbol{\eta}}_i
  -\mathbf{n}\big(\mathbf{t}^{-1}(\boldsymbol{\eta}_i)\big)\Big\|_2^2 ,
\end{align*}
where the time derivatives $\dot{\boldsymbol{\eta}}_i$ are approximated with finite differences from data (our implementation is using a sixth-order central stencil). This is a nonlinear least-squares problem in the free coefficients $(\mathbf{c}^{\mathbf{n}},\mathbf{c}^{\mathbf{t},\star})$, which we solve using Gauss--Newton, with initial condition $\mathbf{c}^{\mathbf{n}},\mathbf{c}^{\mathbf{t},\star}=\mathbf{0}$. Finally, we recover the coefficients $\mathbf{c}^{\mathbf{t}}$ of the forward transformation $\mathbf{t}$ by regression on the following linear least-squares problem:
\begin{align*}
  \mathbf{c}^{\mathbf{t}}=\operatorname*{argmin}_{\mathbf{c}^{\mathbf{t}}}\sum_{i}\Big\|
  \boldsymbol{\eta}_i-\mathbf{t}(\mathbf{t}^{-1}(\boldsymbol{\eta}_i;\mathbf{c}^{\mathbf{t},\star});\mathbf{c}^{\mathbf{t}})\Big\|_2^2.
\end{align*}

\begin{remark}
  The finite-difference estimate for the velocities $\dot{\boldsymbol{\eta}}_i$ is sufficiently accurate provided that the sampling time $\Delta t$ is small relative to the fastest timescale of the SSM dynamics. If this is not the case, we can instead use the discrete-time formulation of SSM theory as described in Appendix~\ref{app:discrete_dynamical_systems}.
\end{remark}

\begin{remark}\label{rem:nf_rescaling}
As in Step~(i.1), the raw coordinates entering the polynomial expansions
\eqref{eqn:expansion_for_n}--\eqref{eqn:expansion_for_t} must be conditioned
before coefficients are fitted: unless the data amplitudes
$\boldsymbol{\zeta}_i=\mathbf{W}^{-1}\boldsymbol{\eta}_i$ are close to unity, the
monomials $\boldsymbol{\phi}_k(\boldsymbol{\zeta})$, $2\le k\le M$, span many
orders of magnitude and the least-squares problem of (c) becomes severely
ill-conditioned. For simplicity we thus use the rescaling $\widetilde{\mathbf{z}} = \mathbf{z}/s$, $s^{2} = \frac{1}{d\,N}\sum_{i=1}^{N}
          \big\|\mathbf{W}^{-1}\boldsymbol{\eta}_i\big\|_2^2$, which commutes with $\boldsymbol{\Lambda}$ and with every $\widehat{S}$. The near-resonant support $\mathcal{I}_\delta$ of Step~(ii.a)
and the equivariant bases $\mathbf{B}_k^{\mathbf{n}},\mathbf{B}_k^{\mathbf{t}}$
of Step~(ii.b) carry over to the scaled frame verbatim.
\end{remark}

\paragraph{Summary of the eSSM algorithm} For clarity of presentation we summarise the above procedure in Algorithm~\ref{alg:eSSM}. The analogous discrete-time version of the algorithm is presented in Algorithm~\ref{alg:discrete_eSSM} in Appendix~\ref{app:discrete_dynamical_systems}. We note that, by construction, the resulting SSM reduction is exactly $\mathcal{G}$-equivariant.

\begin{algorithm}[h!]
\caption{eSSM reduction method}
\label{alg:eSSM}
\begin{algorithmic}[1]
\Statex \textbf{Inputs:} snapshot matrix $\mathbf{Y}\in\mathbb{R}^{n\times N}$, symmetry group
$\mathcal{G}$, reduced dimension $d$, manifold degree $M$, time step $\Delta t$ in $\mathbf{Y}$, resonance tolerance $\delta$, $M_{\text{ROD}}$ for reduced-order dynamics.
\Statex
\Statex \emph{Step (i): Spectral submanifold identification}
\State \textbf{Compute} the Cholesky factorisation $\mathbf{P}_{\mathcal{G}} = \mathbf{L}\mathbf{L}^\top$
of $\mathbf{P}_{\mathcal{G}} = \frac{1}{|\mathcal{G}|}\sum_{S\in\mathcal{G}} S^\top S$.
\State \textbf{Whiten} $\widetilde{\mathbf{Y}} = \mathbf{L}^\top \mathbf{Y}$ and \textbf{stack}
$\widetilde{\mathbf{Y}}_{\mathcal{G}} = \big[\,\widetilde{S}\widetilde{\mathbf{Y}}\,\big]_{S\in\mathcal{G}}$.
\State \textbf{Compute} the $d$-truncated SVD
$\widetilde{\mathbf{Y}}_{\mathcal{G}} \approx \widetilde{\mathbf{U}}_0\widetilde{\boldsymbol{\Sigma}}\widetilde{\mathbf{V}}^\top$.
\State \textbf{Freeze} 
$\mathbf{R}_0(S)=\mathbf{V}_{1,0}^{\top}S\mathbf{U}_{1,0}$, $S\in\mathcal{G}$ and \textbf{set} $\mathbf{D} = \operatorname{diag}(\widetilde{\boldsymbol{\Sigma}})/\sqrt{|\mathcal{G}|N}$.
\For{$k=1,\dots,M$}
  \State \textbf{Compute} the nullspace basis $\mathbf{B}_k$ of $\mathbf{M}_k$ via SVD and whiten it, $\widetilde{\mathbf{B}}_k=(\mathbf{I}_{N_k}\otimes\mathbf{L}^\top)\mathbf{B}_k$.
\EndFor
\State \textbf{Warm start} $\mathbf{c} \gets \widetilde{\mathbf{B}}_1^\top\operatorname{vec}(\widetilde{\mathbf{U}}_0)$,
\quad $\mathbf{w}_k \gets \mathbf{0}$, $k=2,\dots,M$
\State \textbf{Solve} the constrained nonlinear least-squares problem
\eqref{eqn:reduced_constrained_ls}--\eqref{eqn:reduced_constraints}.
\State \textbf{Recover} $\mathbf{U}_1 = \mathbf{L}^{-\top}\widetilde{\mathbf{U}}_1$,
\quad $\mathbf{V}_1 = \mathbf{L}\widetilde{\mathbf{U}}_1$,
\quad $\mathbf{W}_k = \mathbf{L}^{-\top}\widetilde{\mathbf{W}}_k$.
\Statex
\Statex \emph{Step (ii): Compute extended normal form on $\mathcal{W}(E)$}
\State \textbf{Form} the orbit-augmented reduced snapshots $\boldsymbol{\Xi}_{\mathcal{G}}=[\,S|_E\boldsymbol{\Xi}\,]_{S\in\mathcal{G}}$, $\boldsymbol{\Xi}'_{\mathcal{G}}=[\,S|_E\boldsymbol{\Xi}'\,]_{S\in\mathcal{G}}$.
\State \textbf{Estimate} the linearised one-step map $\mathbf{B}_{\Delta t}=\boldsymbol{\Xi}'_{\mathcal{G}}\boldsymbol{\Xi}_{\mathcal{G}}^{\dagger}$, and \textbf{set} $\mathbf{B}=\tfrac{1}{\Delta t}\log\mathbf{B}_{\Delta t}$.
\For{$k=2,\dots,M$}
  \State \textbf{Select} the resonant support $\mathcal{I}_\delta=\{(j,\mathbf{m}):|\operatorname{Im}(\mathbf{m}\cdot\boldsymbol{\lambda}-\lambda_j)|<\delta\}$.
  \State \textbf{Compute} the equivariant bases $\mathbf{B}_k^{\mathbf{n}},\mathbf{B}_k^{\mathbf{t}}$ of $\mathbf{M}_k[\widehat{S}]$ restricted to $\operatorname{supp}(\mathcal{I}_\delta), \operatorname{supp}(\mathcal{I}_\delta^{\,c})$.
\EndFor
\State \textbf{Fit} $(\mathbf{c}^{\mathbf{n}},\mathbf{c}^{\mathbf{t},\star})$ on the conjugacy residual \eqref{eqn:conjugacy}, initialised at $\mathbf{0}$.
\State \textbf{Recover} $\mathbf{c}^{\mathbf{t}}$ by linear least squares and \textbf{assemble} $\mathbf{N}_k,\mathbf{T}_k,\mathbf{T}_k^{\star}$ via \eqref{eqn:equivariant_basis_N_and_T}.\vspace{-0.2cm}
\Statex
\Statex \textbf{Outputs:} SSM parametrisation
$(\mathbf{U}_1,\mathbf{V}_1,\{\mathbf{W}_k\}_{k=2}^{M},\mathbf{D})$, and reduced
dynamics $(\mathbf{W},\boldsymbol{\Lambda},
\{\mathbf{N}_k,\mathbf{T}_k,\mathbf{T}^{\star}_k\}_{k=2}^{M_{\text{ROD}}})$ on
$\mathcal{W}(E)$ in extended normal form.
\end{algorithmic}
\end{algorithm}

\subsection{Simulation using the reduced order model} With the output of Algorithm~\ref{alg:eSSM}, we can then simulate the reduced-order dynamics from a given initial condition $\mathbf{y}_0\in\mathbb{R}^n$ as follows. Firstly, $\mathbf{y}_0$ is projected to the reduced coordinates through the chart,
$\boldsymbol{\eta}_0=\mathbf{V}_1^\top\mathbf{y}_0$, and mapped to normal-form
coordinates, $\mathbf{z}_0=\mathbf{t}^{-1}(\boldsymbol{\eta}_0)$, using
\eqref{eqn:expansion_for_t}. Secondly, we integrate the $d$-dimensional normal-form dynamics $\dot{\mathbf{z}}=\mathbf{n}(\mathbf{z})$ from \eqref{eqn:expansion_for_n}. Thirdly, the trajectory is mapped back to the original coordinates $\boldsymbol{\eta}(t)=\mathbf{t}(\mathbf{z}(t))$ via \eqref{eqn:expansion_for_t} on $E$. Finally, the full state is recovered through the graph parametrisation \eqref{eqn:graph_parametrisation_data},\vspace{-0.2cm}
\begin{align*}
  \mathbf{y}(t)
  = \mathbf{U}_1\boldsymbol{\eta}(t)
  + \sum_{k=2}^{M}\mathbf{W}_k\,
    \boldsymbol{\phi}_k\big(\mathbf{D}^{-1}\boldsymbol{\eta}(t)\big).
\end{align*}\vspace{-1.2cm}\\

\section{Numerical examples}\label{sec:examples}

Following the above exposition of the eSSM reduction method, we now present several examples comparing this new, equivariant method to standard SSM reduction and similar data-driven methods introduced in earlier work. All of the following experiments were conducted on an Apple M2 Max with 64 GB RAM.

\subsection{Example 1: chain of oscillators}

In this first example we apply the eSSM reduction method to the chain of oscillators described in Example~\ref{ex:chain_of_oscillators}. We use this simple example to understand how much the free parameter count can be reduced by enforcing equivariance and to provide a direct comparison against a non-equivariant SSM reduction method. For this we compare the following three methods:
\begin{itemize}
  \item \textbf{eSSM:} Our new method as described in \S\ref{sec:data_driven_ssm_reduction_with_equivariance} and Algorithm~\ref{alg:eSSM}. The corresponding Python implementation used in the following examples is available at \url{https://github.com/GeorgAUT/eSSM}.
  \item \textbf{SSMLearn:} The original Matlab implementation of the data-driven SSM discovery method introduced in \cite{SSMLearnpaper22}, available at \url{https://github.com/haller-group/SSMLearn}.
  \item \textbf{SSMLearnPy:} The original Python implementation of the SSMLearn method, available at \url{https://github.com/haller-group/SSMLearnPy}.
  \item \textbf{SSMLearn (Python):} For fairness of runtime comparison, we also compare against our eSSM Python implementation with trivial group $\mathcal{G}=\{I\}$ which effectively reduces to the original SSMLearn algorithm.
\end{itemize}
We will compare the performance of these methods in terms of accuracy, number of free parameters and wall-clock time.\vspace{-0.2cm}

\paragraph{Experimental setup}
We consider the chain of Example~\ref{ex:chain_of_oscillators} with up to
$100$ masses, i.e.\ state-space dimension $n = 200$.  The dataset consists of
four trajectories released from four different initial conditions integrated
over $t\in[0,2000]$ and sampled at $\Delta t = 0.05$.  The first
$500$ samples of every trajectory ($25$ time units) are discarded as the
off-manifold transient towards the slow SSM. We train the SSM methods on two trajectories and evaluate on two unseen trajectories at new initial conditions.  The symmetry
group supplied to eSSM is the parity group
$\mathcal{G}=\{\mathbf{I},-\mathbf{I}\}$ of
Example~\ref{ex:chain_of_oscillators}. Configurations of all methods are chosen identically where possible: reduced dimension $d = 2$, manifold order $M$, and resonance tolerance $\delta = 10^{-4}$. The largest singular values of the orbit-augmented training data and a fitted spectral subspace and SSM are shown in Figure~\ref{fig:fitting_chain_oscillators}.\vspace{-0.2cm}

\begin{figure}[h!]
    \centering
    \begin{subfigure}[t]{0.48\textwidth}
        \centering
        \includegraphics[width=0.85\linewidth]{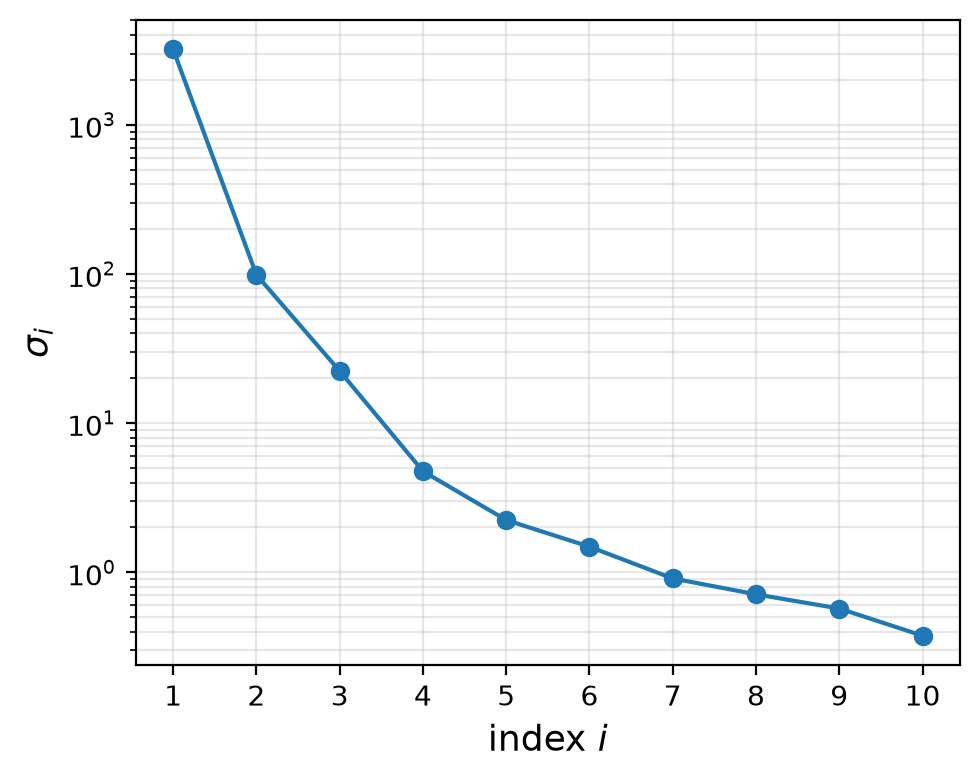}
        \caption{Leading singular values $\sigma_i$ of the orbit-augmented
        training data.}
        \label{fig:chain_spectral_gap}
    \end{subfigure}
    \hfill
    \begin{subfigure}[t]{0.48\textwidth}
        \centering
        \includegraphics[width=\linewidth]{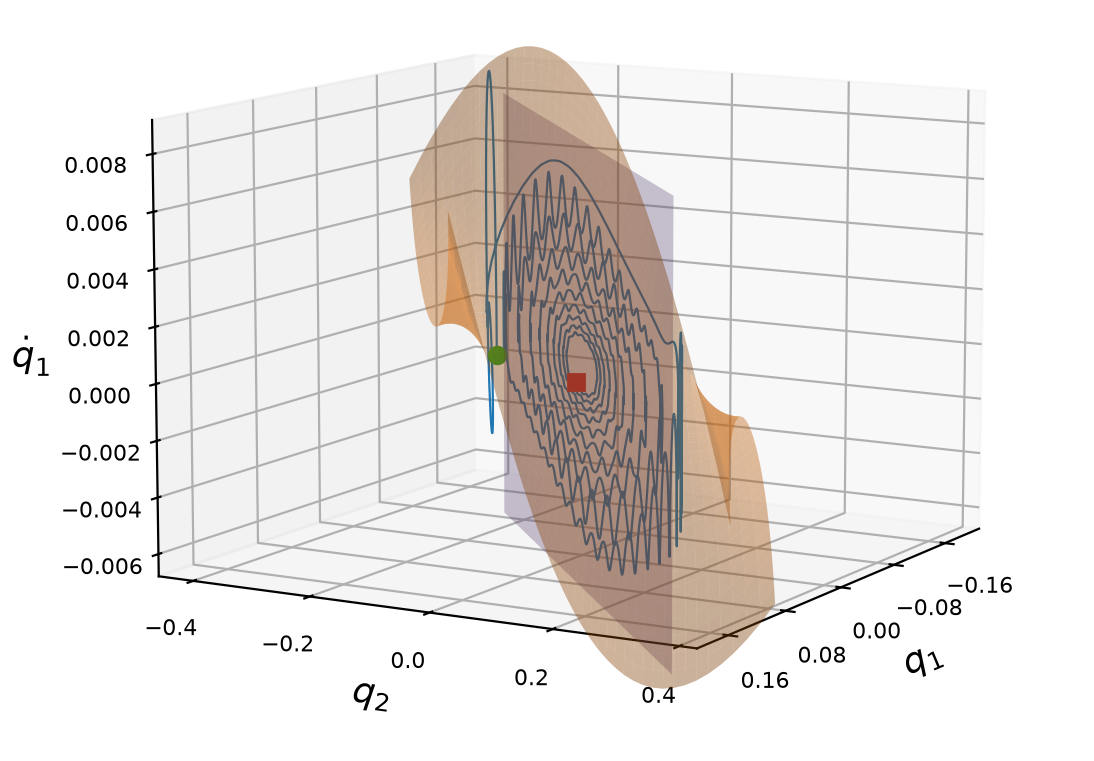}
        \caption{Fitted spectral subspace (purple plane) and SSM (orange surface), together with a
        training trajectory (blue).}
        \label{fig:chain_manifold_fit}
    \end{subfigure}
    \caption{Manifold identification for the oscillator chain ($n = 100$,
    $d = 2$, $M = 3$).}
    \label{fig:fitting_chain_oscillators}
\end{figure}

All iterative fits are run to first-order
optimality tolerance $10^{-9}$ with an iteration cap of $2000$, where possible (in the SSMLearn Matlab implementation these handles were not fully available as is discussed further below). Our quality metric is the normalised mean trajectory error (NMTE) on the test set,
\begin{align}\label{eqn:nmte}
  \mathrm{NMTE}
  = \frac{\big\|\hat{\mathbf{y}}_i-\mathbf{y}_i\big\|_2}
    {\max_{i}\|\mathbf{y}_i\|_2}\times 100\,\%,
\end{align}
where $\hat{\mathbf{y}}_i$ is the model prediction of the $i$-th test trajectory $\mathbf{y}_i$.

\paragraph{Experimental results} In our first experiment we examine the behaviour of the methods as the manifold order $M$ is varied. The results of this experiment can be seen in Figure~\ref{fig:chain_prediction_M}. As demonstrated in Example~\ref{ex:equivariant_taylor_expansion_damped_oscillator_chain} the parity symmetry of the chain of oscillators eliminates all even-degree monomials from the Taylor expansions of the SSM parametrisation and the reduced dynamics.  This reduces the number of free parameters in the model significantly (cf. Table~\ref{tab:chain_oscillators_results}). 
In practice, we observe that this reduced parameter count leads to a significant reduction in wall-clock time of the fit (cf. Figure~\ref{fig:chain_prediction_M_time}), while maintaining the accuracy of the reduced dynamics on the test trajectory (cf. Figure~\ref{fig:chain_prediction_M_nmte}).
\begin{table}[h!]
  \centering
  \begin{tabular}{rrrr}
\toprule
$M$ & SSMLearn & eSSM & reduction \\
\midrule
1 &  202 &  202 & 0.0\% \\
2 &  508 &  202 & 60.2\% \\
3 &  916 &  610 & 33.4\% \\
4 & 1426 &  610 & 57.2\% \\
5 & 2038 & 1222 & 40.0\% \\
6 & 2752 & 1222 & 55.6\% \\
\bottomrule
\end{tabular}

\caption{Total number of fitted parameters in the case $n=100, d=2$ as a function of $M=M_{\text{ROD}}$.}
\label{tab:chain_oscillators_results}
\end{table}

\begin{figure}[h!]
\centering
\begin{subfigure}[t]{0.5\textwidth}
    \centering
    \includegraphics[width=\linewidth]{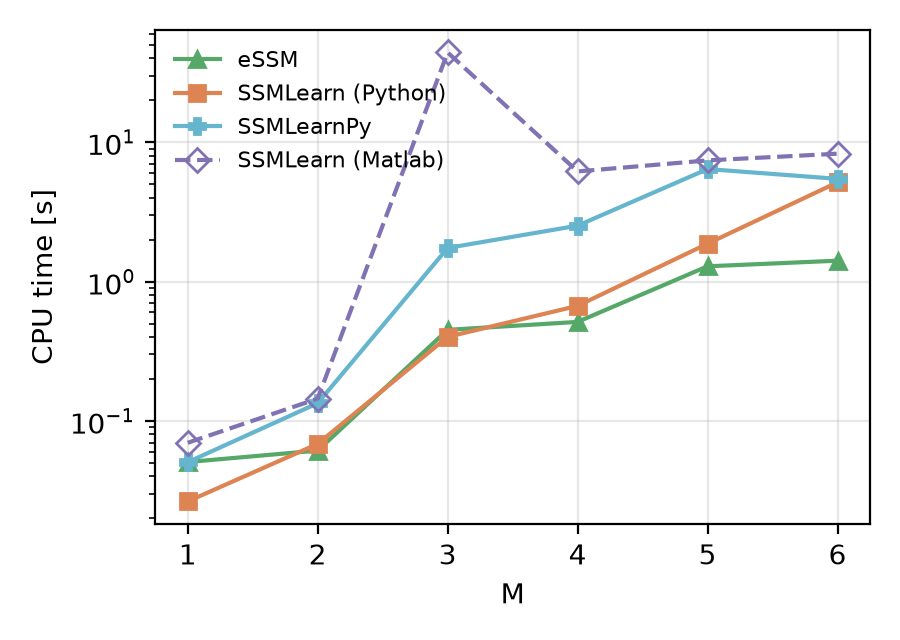}
    \caption{Wall-clock time of the full fit.}
    \label{fig:chain_prediction_M_time}
\end{subfigure}%
\begin{subfigure}[t]{0.5\textwidth}
    \centering
    \includegraphics[width=\linewidth]{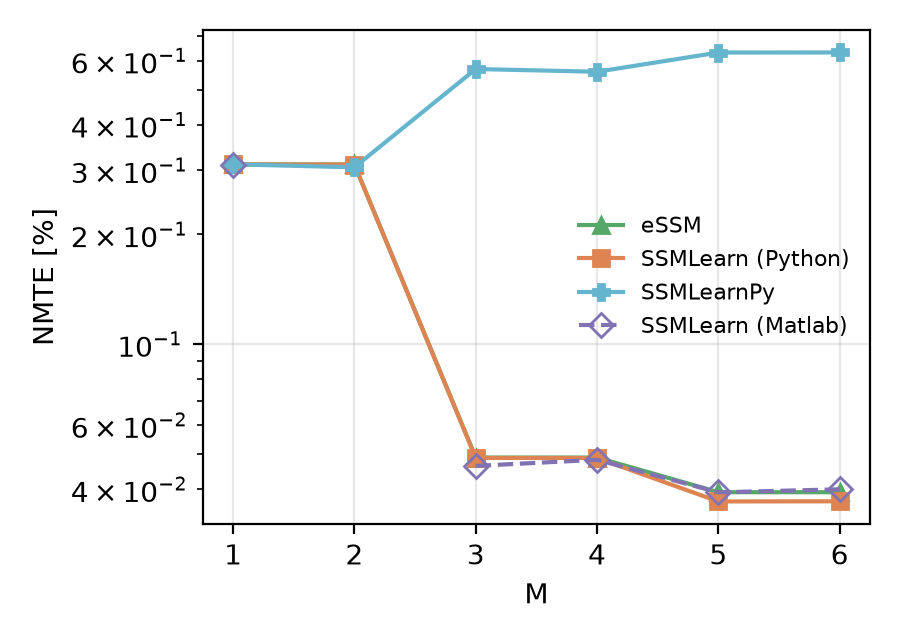}
    \caption{NMTE on the held-out trajectories.}
    \label{fig:chain_prediction_M_nmte}
\end{subfigure}
\caption{Prediction benchmark versus the manifold order $M$ ($n = 100$,
$d = 2$, $M_{\text{ROD}} = M$).}
\label{fig:chain_prediction_M}
\end{figure}

In our second experiment we fix the manifold order to $M = 3$ and vary the ambient dimension $n$ of the chain of oscillators. The results of this experiment can be seen in Figure~\ref{fig:chain_prediction_n}. As expected, the wall-clock time of the fit increases with increasing ambient dimension (cf. Figure~\ref{fig:chain_prediction_n_time}), while the accuracy of the reduced dynamics on the test trajectory remains largely unaffected (cf. Figure~\ref{fig:chain_prediction_n_nmte}, noting the scale of y-axis). As in the previous experiment, the reduced parameter count of the eSSM method leads to a significant reduction in wall-clock time of the fit, while maintaining the accuracy of the reduced dynamics on the test trajectory.

\begin{figure}[h!]
    \centering
\begin{subfigure}[t]{0.5\textwidth}
        \centering
        \includegraphics[width=\linewidth]{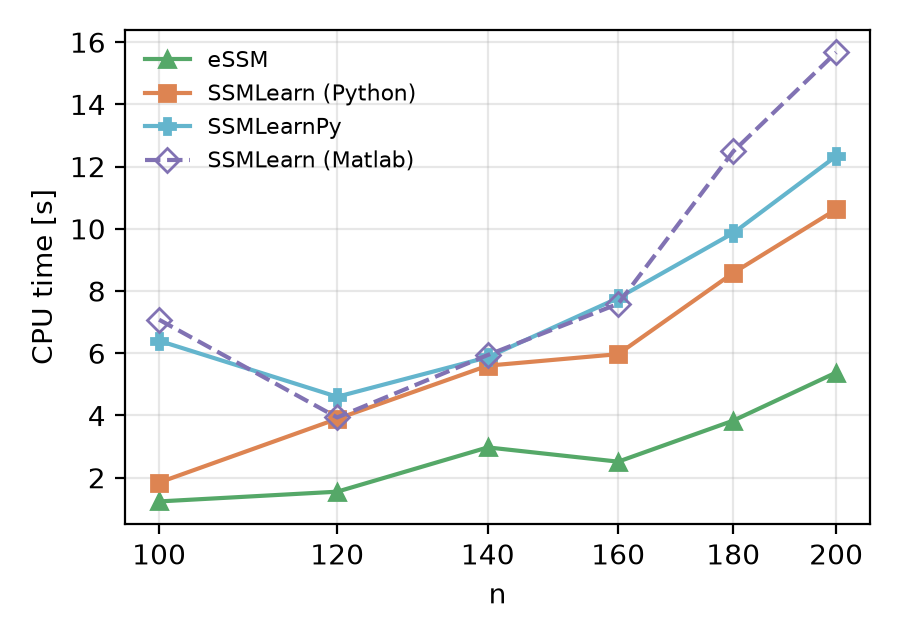}
        \caption{Wall-clock time.}
        \label{fig:chain_prediction_n_time}
    \end{subfigure}%
   \begin{subfigure}[t]{0.5\textwidth}
        \centering
        \includegraphics[width=\linewidth]{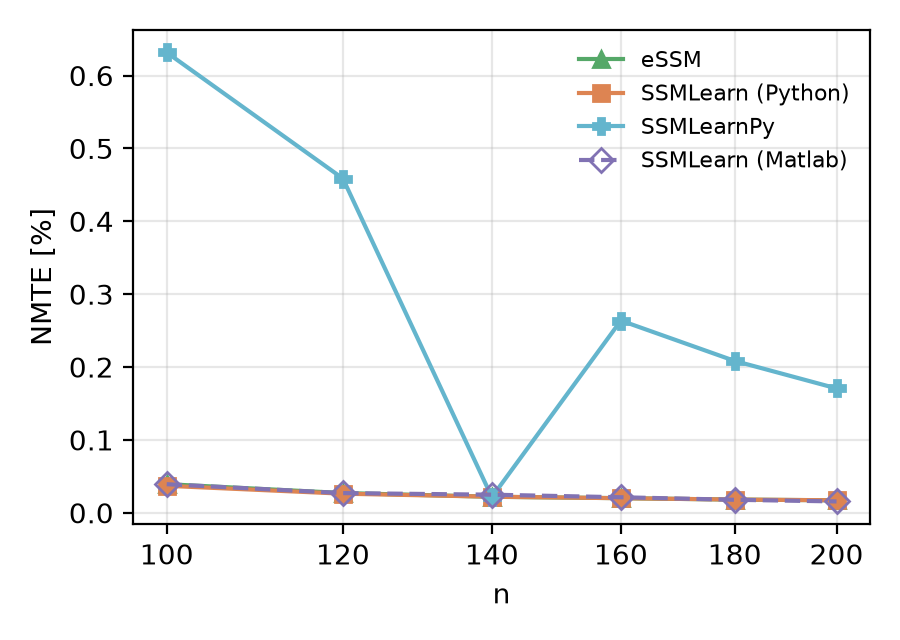}
        \caption{NMTE on the test set.}
        \label{fig:chain_prediction_n_nmte}
    \end{subfigure}
    \caption{Performance of the methods as a function of the ambient dimension
    $n$ at fixed $M = 3$ ($d = 2$, $M_{\text{ROD}} = 3$).}
    \label{fig:chain_prediction_n}
\end{figure}

\newpage\subsection{Example 2: dissipative shallow-water equations on the sphere}\label{sec:swe_sphere}

Our second example is a two-dimensional PDE whose symmetry is inherited from the geometry of the underlying domain: the viscous shallow-water equations on a rotating sphere, $S^2 = \{\mathbf{x}\in\mathbb{R}^3 : |\mathbf{x}| = 1\}$, in the formulation of \cite{galewsky2004initial},
\begin{align}\label{eqn:swe_sphere}
  \partial_t \mathbf{u} + \nu\Delta^2\mathbf{u} + g\nabla h
    + f\,\hat{\mathbf{k}}\times\mathbf{u} + \gamma\,\mathbf{u}
    &= -(\mathbf{u}\cdot\nabla)\mathbf{u},\\
  \partial_t h + \nu\Delta^2 h + H\nabla\cdot\mathbf{u}
    &= -\nabla\cdot(h\,\mathbf{u}),\nonumber
\end{align}
where $\mathbf{u}$ is the tangential velocity field, $h$ the perturbation of the surface fluid about the constant mean depth $H$, $f = 2\Omega\sin\varphi$ the Coriolis parameter at latitude $\varphi$, $g$ is the gravitational constant, $\hat{\mathbf{k}}$ is the outward unit normal, and $\nabla$, $\nabla\cdot$, $\Delta$ the intrinsic surface differential operators. The state of rest $(\mathbf{u}, h) = (\mathbf{0}, 0)$ is a fixed point and the geometry induces a natural symmetry group of rotations about the polar axis. Our time series data are generated by observing the height perturbation $h$ and the relative vorticity $\zeta = \hat{\mathbf{k}}\cdot(\nabla\times\mathbf{u})$ at a finite number of $36$ sensor locations on the sphere (cf. Figure~\ref{fig:swe_sensor_locations}).

\begin{figure}[h!]
\centering
    \centering
    \includegraphics[width=0.53\linewidth]{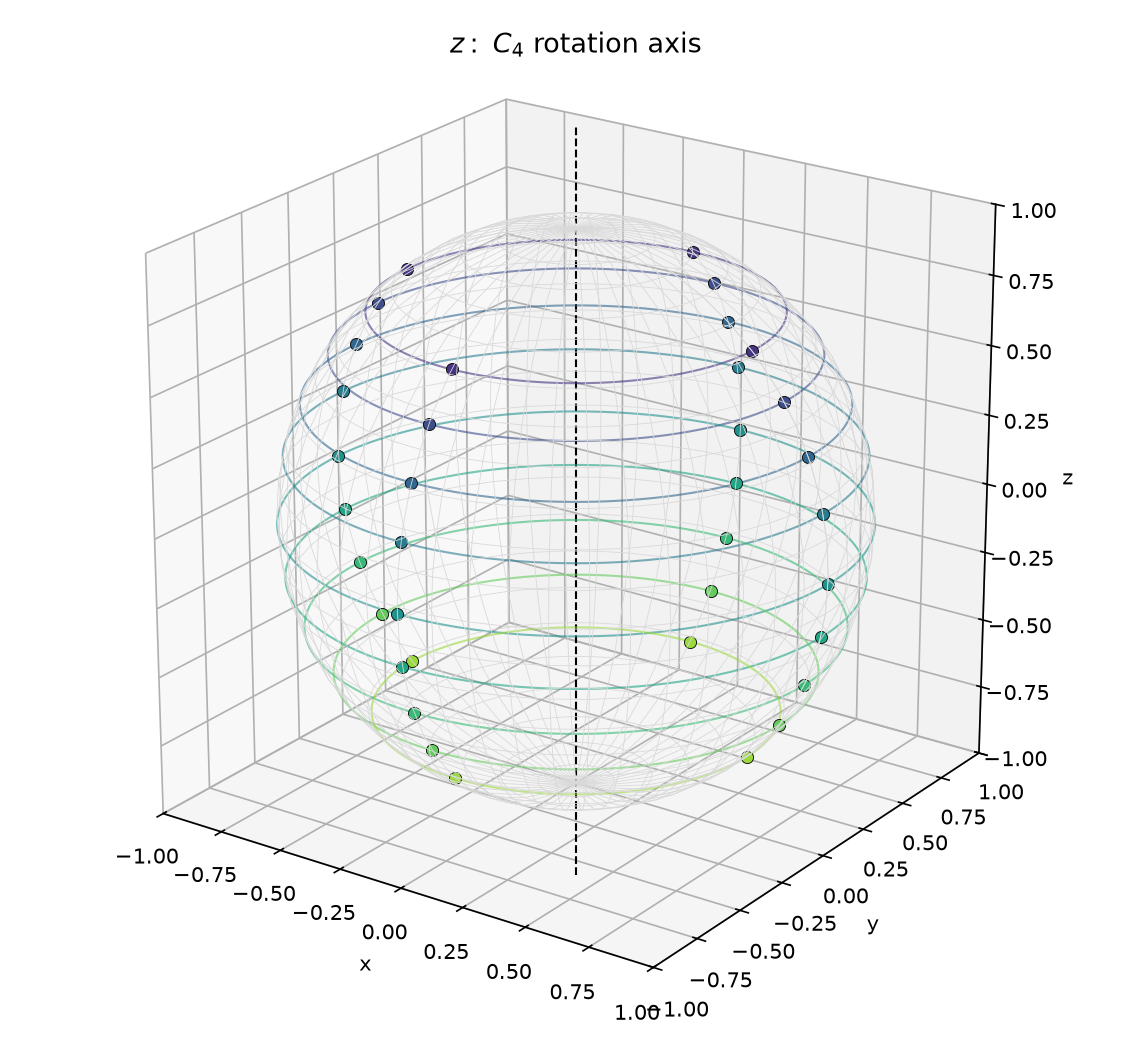}
    \caption{Sensor locations in our SWE experiment at four longitudes and nine latitudinal rings.}
    \label{fig:swe_sensor_locations}
\end{figure}
The system \eqref{eqn:swe_sphere} is equivariant under the group of rotations about the polar axis, which implies that the observations inherit a discrete cyclic symmetry, in the discrete group of rotations $\mathcal{G}=C_4$, from our sensor placement.

\paragraph{Experimental setup}
Data is generated with the spectral solver Dedalus \cite{dedalus2020}, using the publicly available spherical shallow-water example \verb|vp_sphere_shallow_water| with minor modifications, including the implementation of the drag term and the modified initial conditions
\begin{align*}
  h_0 &= \tilde{h}_0 - \int_{{S}^2} \tilde{h}_0 \, d\sigma, \quad \tilde{h}_0= A\Big[\cos\varphi\cos\lambda
  + \varepsilon\,e^{-(1 - \cos d(\lambda,\varphi))/w}\Big]\\
  \qquad \mathbf{u}_0 &= \mathbf{0},
\end{align*}
with $d$ the great-circle distance to $(\lambda_0,\varphi_0) = (0, \pi/6)$ and $A = 0.4\,H$. Our parameter choices are $\Omega = 0.2625$, $g = 19.95$, $H = 1.570\times 10^{-3}$, $\nu = 1.75\times 10^{-3}$, $\gamma = 2\times 10^{-3}$, $\varepsilon = 0.15$, $w = 0.15$, and the spectral method resolution used in Dedalus is $(N_\lambda, N_\theta) = (128, 64)$. A visualisation of the states of this system can be seen in Figure~\ref{fig:swe_setup}.

\begin{figure}[h!]
\centering
\begin{subfigure}[t]{0.5\textwidth}
    \centering
    \includegraphics[width=0.9\linewidth]{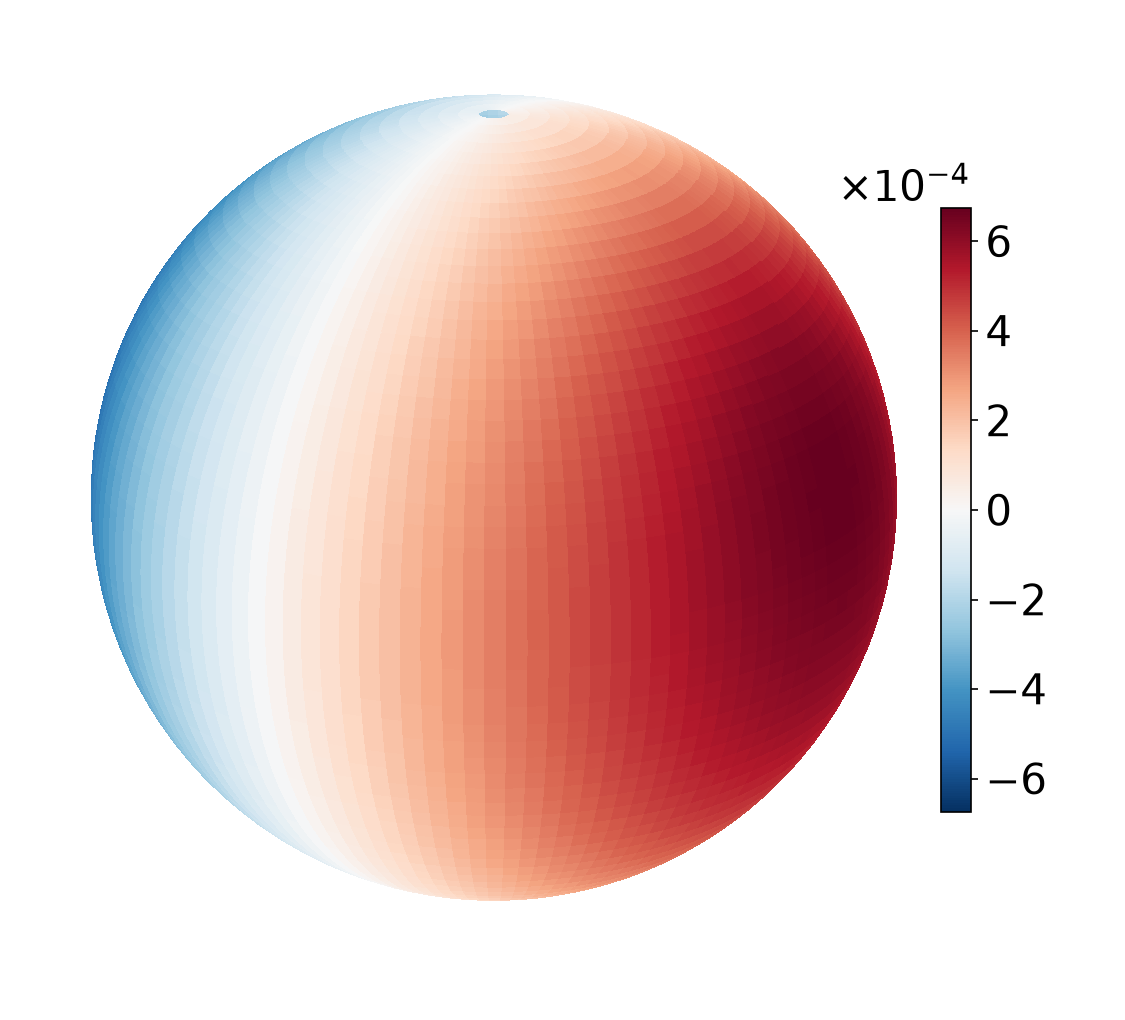}
    \caption{$h$ at $t=0$.}
\end{subfigure}%
\begin{subfigure}[t]{0.5\textwidth}
    \centering
    \includegraphics[width=0.9\linewidth]{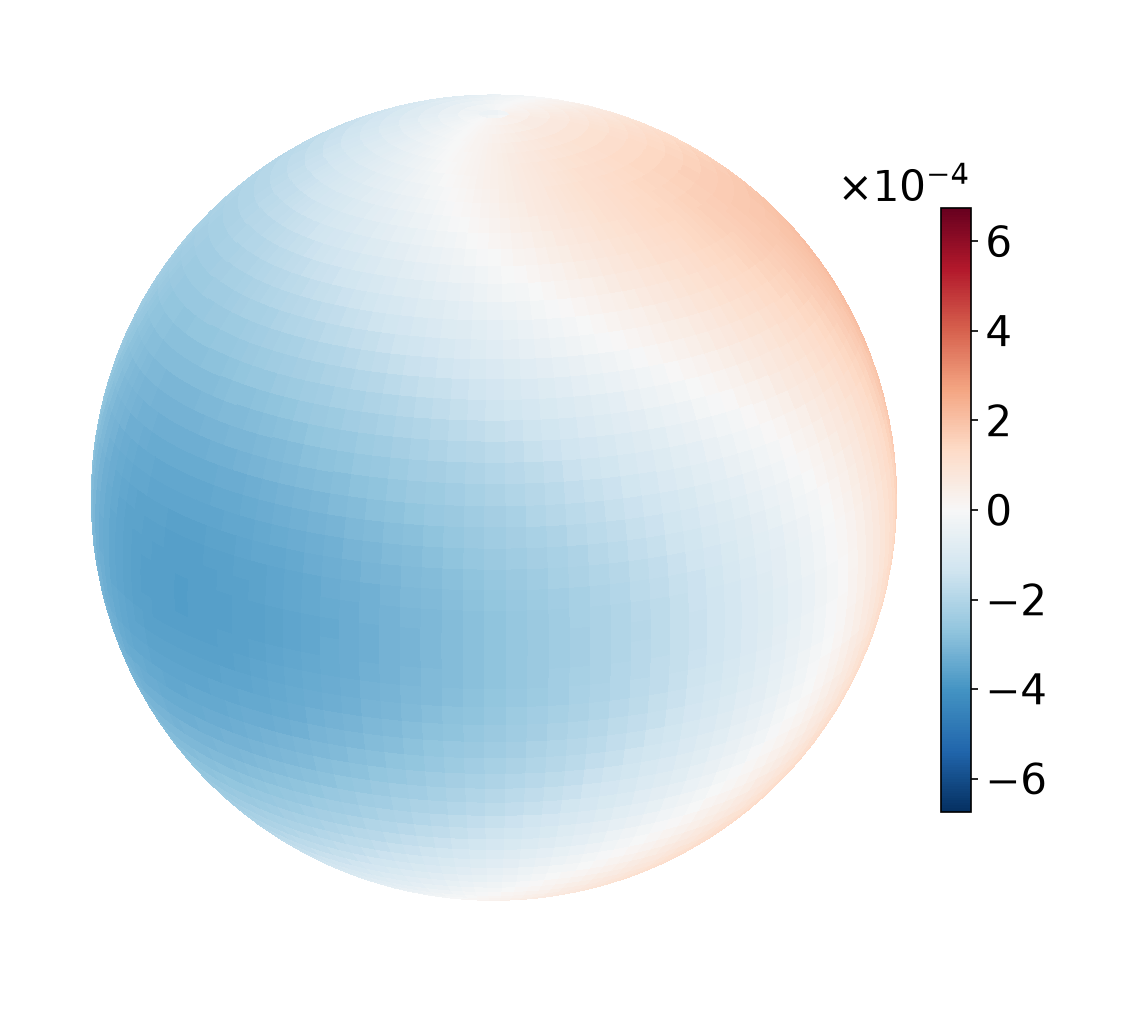}
    \caption{$h$ at $t=8$.}
\end{subfigure}
\caption{Visualisation of the $h$-perturbation in \eqref{eqn:swe_sphere}.}
\label{fig:swe_setup}
\end{figure}

The trajectory is integrated over $t\in[0, 400]$ and sampled at $\Delta t = 0.25$; the first $600$ steps are discarded as the off-manifold transient towards the slow SSM. Our observations live in $\mathbb{R}^{72}$ (two scalar fields) and $\mathcal{G} \cong C_4$. Configurations of all methods are again chosen identically where possible. Given the quadratic nature of the nonlinearity in \eqref{eqn:swe_sphere} the reduced-dynamics order $M_{\text{ROD}}$ was fixed to $2$ throughout this example. We fit the SSM dynamics on a single trajectory as above, and evaluate on two metrics (using the NMTE \eqref{eqn:nmte} as in Example 1):
\begin{itemize}
  \item \textbf{Self:} corresponding to the training trajectory, which is used to assess the accuracy of the reduced dynamics on the SSM.
  \item \textbf{Rotated:} corresponding to a trajectory obtained by applying a 90-degree rotation to the training trajectory, which is used to assess the accuracy of the reduced dynamics on the SSM under symmetry transformations.
\end{itemize}
Given the more complex nature of this example we commence with a sweep over the manifold dimension $d$ to identify a suitable reduced dimension for the SSM, fixing $M=2$ given the quadratic nature of the nonlinearity in \eqref{eqn:swe_sphere}. The results of this sweep can be seen in Figure~\ref{fig:swe_prediction_d}. We observe that the NMTE appears to be smallest at $d=6$ thus suggesting that $d = 6$ is a suitable reduced dimension for the SSM. As in Example 1 we notice that eSSM achieves comparable accuracy to SSMLearn with a significantly reduced CPU time. We note in this example SSMLearn Matlab is underperforming in terms of accuracy (with $d=8$ the method did not converge). This is most likely due to the fact that the Matlab implementation of SSMLearn enforces $M_{\text{ROD}} \geq 3$ (which is a sensible constraint for exact normal forms) and thus is unable to fit the reduced dynamics on the extended normal form with $M_{\text{ROD}} = 2$ directly. Apparently, this dynamics fitting problem becomes badly conditioned when $M_{\text{ROD}}\geq 3$ thus leading to the poor fit observed in Figure~\ref{fig:swe_prediction_d_orbit}.

\begin{figure}[h!]
\centering
\begin{subfigure}[t]{0.5\textwidth}
    \centering
    \includegraphics[width=\linewidth]{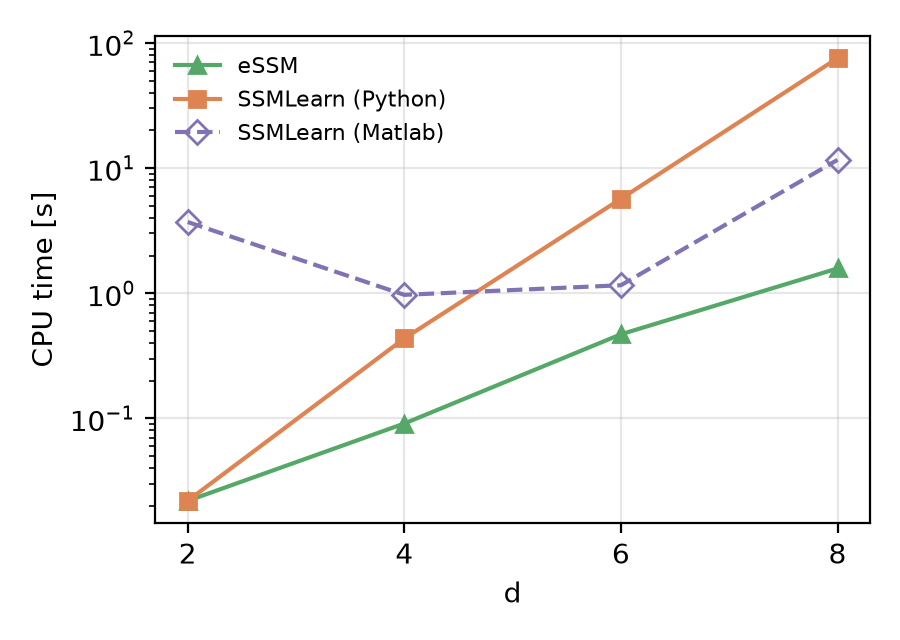}
    \caption{Wall-clock time of the full fit.}
\end{subfigure}%
\begin{subfigure}[t]{0.5\textwidth}
    \centering
    \includegraphics[width=\linewidth]{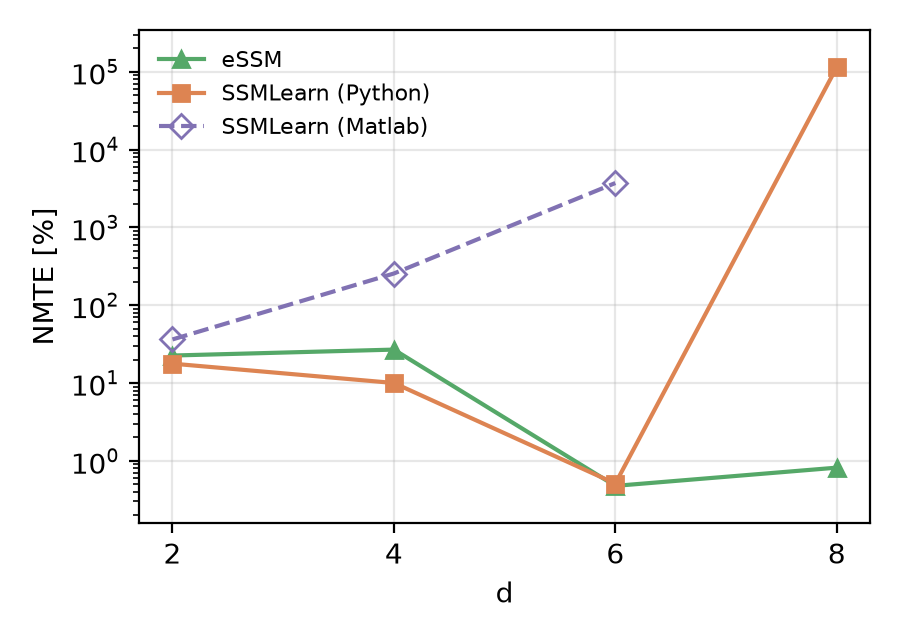}
    \caption{Prediction error on the training trajectory.}
    \label{fig:swe_prediction_d_orbit}
\end{subfigure}
\caption{Performance of the methods as a function of the reduced dimension $d$
at fixed $M=2$.}
\label{fig:swe_prediction_d}
\end{figure}

In our second experiment we fix the reduced dimension to $d = 6$ and vary the manifold order $M$ (with $M_{\text{ROD}} = 2$). The results of this experiment can be seen in Figure~\ref{fig:swe_prediction_M}. In this experiment the fourth order symmetry group $C_4$ leads to a nearly 75\% reduction in the number of free parameters at every manifold order (cf. Table~\ref{tab:swe_results}).

\begin{table}[h!]
\centering
\begin{tabular}{rrrr}
\toprule
$M$ & SSMLearn & eSSM & reduction \\
\midrule
1 &   566 &  143 & 74.7\% \\
2 &  2078 &  521 & 75.0\% \\
3 &  6110 & 1529 & 75.0\% \\
4 & 15182 & 3797 & 75.0\% \\
5 & 33326 & 8333 & 75.0\% \\
\bottomrule
\end{tabular}
\caption{Total number of fitted parameters for
the shallow-water example ($n = 72$, $d = 6$) as a function of the manifold
order $M$.}
\label{tab:swe_results}
\end{table}

The result in terms of practical performance can be seen in Figure~\ref{fig:swe_prediction_M}. We observe that the reduced parameter count leads to a significant reduction in wall-clock time of the fit at matched prediction accuracy on the training trajectory (Figure~\ref{fig:swe_prediction_M}), in particular we observe a roughly 50\% cost reduction throughout and a more significant reduction at $M=1,2$ where the unconstrained SSMLearn method struggles to fit the reduced dynamics on the extended normal form with $M_{\text{ROD}} = 2$ directly. We note that the accuracy of the reduced dynamics on the training trajectory remains largely unaffected by the manifold order $M$ (Figure~\ref{fig:swe_prediction_M_orbit}), however the accuracy of the reduced dynamics on the symmetry-transformed trajectory is significantly improved by enforcing equivariance.\vspace{-0.3cm}

\begin{figure}[h!]
\centering
\begin{subfigure}[t]{0.5\textwidth}
    \centering
    \includegraphics[width=\linewidth]{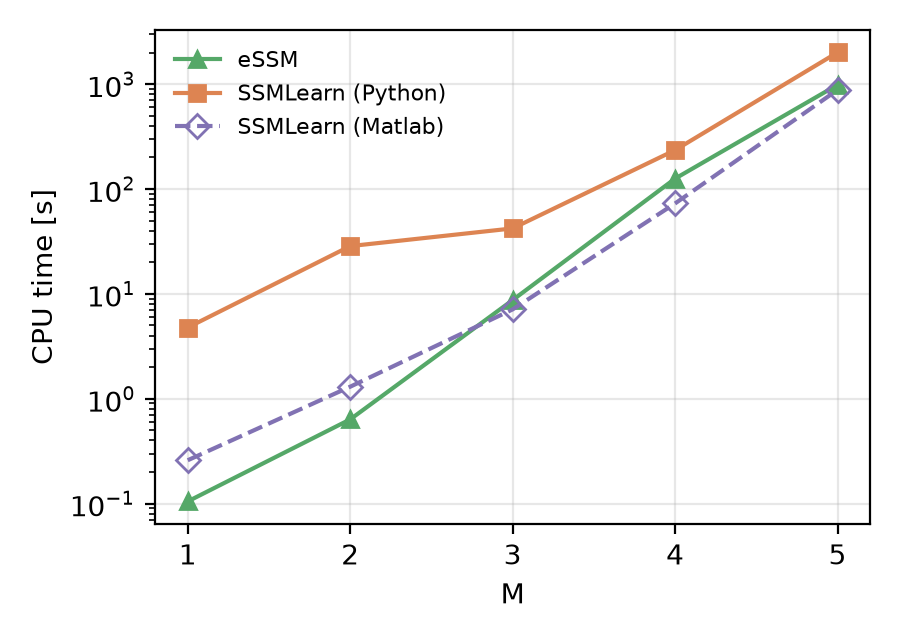}\vspace{-0.2cm}
    \caption{Wall-clock time of the full fit.}
    \label{fig:swe_prediction_M_time}
\end{subfigure}%
\begin{subfigure}[t]{0.5\textwidth}
    \centering
    \includegraphics[width=\linewidth]{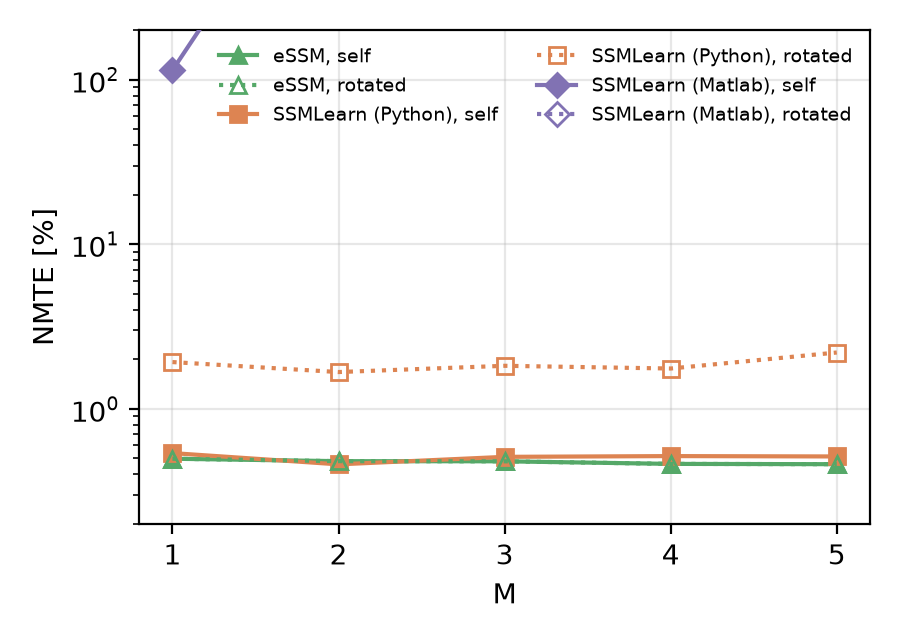}\vspace{-0.2cm}
    \caption{Prediction error on training and rotated trajectories.}
    \label{fig:swe_prediction_M_orbit}
\end{subfigure}
\caption{Prediction benchmark versus the manifold order $M$ ($n = 72$, $d = 6$).}
\label{fig:swe_prediction_M}
\end{figure}

\subsection{Example 3: KS equation on periodic domain (CTF4Science benchmark)}\label{sec:KS_equation}
In this final example we will benchmark the eSSM method on the Kuramoto--Sivashinsky (KS) equation challenge of the CTF4Science project \cite{wyder2026common}. The KS equation,
\begin{align}\label{eqn:KS_equation}
  \partial_t u + u\,\partial_x u + \partial_{xx} u + \mu\,\partial_{xxxx} u = 0,
  \qquad x\in[0, 32\pi],
\end{align}
with periodic boundary conditions, is a canonical example of spatio-temporal chaos in one dimension. The benchmark provides training trajectories of \eqref{eqn:KS_equation} on an $1024$-point grid (time step and initial conditions undisclosed) and evaluates predictions on hidden test data. Strictly speaking, this setting lies outside the scope of our theory: although the origin $u = 0$ is a fixed point of \eqref{eqn:KS_equation}, the CTF4Science data explores a chaotic attractor rather than a decaying transient towards a stable equilibrium. We include this example to test the robustness of SSM-based forecasting of dynamical systems on a standardised benchmark against a broad field of data-driven methods. We focus on the forecasting task (Test~1 of \cite{wyder2026common}), which is scored by a short-time (``weather'') and a long-time (``climate'') metric,
\begin{align}\label{eqn:ctf_scores}
  E_1 = 100\,\big(1 - S_{\mathrm{ST}}\big), \qquad
  E_2 = 100\,\big(1 - S_{\mathrm{LT}}\big),
\end{align}
where $S_{\mathrm{ST}}$ is the relative error of the predicted state over the
$m$ snapshots of the forecast window,
\begin{align*}
  S_{\mathrm{ST}}
  = \frac{\Big(\sum_{i=1}^{m}\big\|\hat{\mathbf{u}}_i-\mathbf{u}_i\big\|_2^2\Big)^{1/2}}
         {\Big(\sum_{i=1}^{m}\big\|\mathbf{u}_i\big\|_2^2\Big)^{1/2}},
\end{align*}
with $\mathbf{u}_i$ the $i$-th test snapshot and $\hat{\mathbf{u}}_i$ its
prediction, and $S_{\mathrm{LT}}$ is the corresponding relative error of the
log power spectral density restricted to the lowest $k_{\max} = 100$
wavenumbers,
\begin{align*}
  S_{\mathrm{LT}}
  = \frac{\Big(\sum_{i=1}^{m}\sum_{|k|\leq k_{\max}}
      \big(\hat{p}_{i,k}-p_{i,k}\big)^2\Big)^{1/2}}
         {\Big(\sum_{i=1}^{m}\sum_{|k|\leq k_{\max}}
      p_{i,k}^2\Big)^{1/2}},
  \qquad
  p_{i,k} = \ln\big|(\mathcal{F}\mathbf{u}_i)_k\big|^2,
\end{align*}
where $(\mathcal{F}\mathbf{u}_i)_k$ denotes the $k$-th discrete Fourier
coefficient of the $i$-th snapshot and $\hat{p}_{i,k}$ the same quantity for
the prediction. The scores are normalised such that a forecast of zeros leads to a score of $0$ in both metrics, and a score of $100$ corresponds to a perfect match with the hidden test data. For further details the reader is referred to \cite{wyder2026common}.

\paragraph{Experimental setup} 

We use the discrete-time eSSM method of Algorithm~\ref{alg:discrete_eSSM} for this forecasting task. The training trajectory is of size $\mathbf{Y}\in\mathbb{R}^{10000 \times 1024}$ and we are asking the model to continue this same trajectory for an additional $1000$ steps. In our setup the state is subsampled to an $n_x$-point spatial grid before the fit and interpolated back to the full grid by trigonometric interpolation. We use a time-delay embedding (cf. \S~\ref{sec:delay_embedding_and_equivariance}) with $q$ copies at lag $q_{\mathrm{lag}}$ steps. Since the CTF measures error directly in the full state space starting at the final step of the training trajectory, the metric is sensitive to the off-manifold residual that arises when projecting the initial condition onto the fitted SSM. To mitigate this, we use a simple exponential decay of the off-manifold residual during the forecast: letting $\boldsymbol{\eta}_0 = \mathbf{V}_1^\top\mathbf{u}_0$ be the initial condition on the fitted SSM, we define the initial off-manifold residual
\begin{align*}
  \mathbf{r}_0 \;:=\; \mathbf{u}_0 \;-\; \mathbf{U}_1\boldsymbol{\eta}_0
                 \;-\; \mathbf{h}(\boldsymbol{\eta}_0),
\end{align*}
i.e.\ the part of $\mathbf{u}_0$ that is not captured by the manifold parametrisation
$\boldsymbol{\eta}\mapsto\mathbf{U}_1\boldsymbol{\eta}+\mathbf{h}(\boldsymbol{\eta})$.
We then correct the forecast $\hat{\mathbf{u}}_i$ at step $i$ by adding the decayed residual
$\mathbf{r}_0\,e^{-i/\tau}$, where $\tau$ is a time constant (in steps) that controls the decay
rate. This ensures that the forecast matches the observed initial condition exactly at $i=0$ and
relaxes onto the SSM prediction over $\mathcal{O}(\tau)$ steps, which removes the initial jump
that would otherwise be introduced by projecting onto the manifold. On the periodic grid, \eqref{eqn:KS_equation} is equivariant under the cyclic group of grid translations and we supply to eSSM the subgroup of shifts by multiples of $2^{k}$ grid points,
\begin{align*}
  \mathcal{G} \cong C_{n_x/2^{k}}, k=0,\dots, \log_2(n_x)-1,
\end{align*}
with $\mathbf{S}$ the elementary cyclic shift. To evaluate our method fairly against the prepopulated CTF4Science leaderboard \cite{wyder2026common}, we follow the original evaluation methodology of \cite{wyder2026common} and use Ray Tune \cite{liaw2018tune} as a hyperparameter tuner to select the optimal configuration of the eSSM method based on an 80/20 split of the training data (Ray Tune does not see the held-out test data used to compute the scores in Table~\ref{tab:ks_results}). We also compare against our Python implementation of SSMLearn (the original Matlab implementation is not compatible with the CTF Python codebase). The optimal configurations found in this way are given in Table~\ref{tab:ks_config}.

\begin{table}[!ht]
  \centering
  \begin{tabular}{lrrrrrrrr}
  \toprule
  & $d$ & $M$ & $M_{\text{ROD}}$ & $n_x$ & $q$ & $q_{\mathrm{lag}}$ & stride & $\mathcal{G}$ \\
  \midrule
  eSSM     & 2 & 4 & 2 & 64  & 2 & 2 & 2  & $C_8$ \\
  SSMLearn (Python) & 2 & 4 & 2 & 64  & 2 & 2 & 2  & $-$ \\
  \bottomrule
  \end{tabular}
  \caption{Ray-Tune-selected configurations ($\tau = 500$ steps for both): $d$
  the SSM dimension, $M$ the manifold order, $M_{\text{ROD}}$ the
  reduced-dynamics order, $n_x$ the subsampled grid size, $q$ delay copies at
  lag $q_{\mathrm{lag}}$, stride is the temporal subsampling of the fit samples,
  and $\mathcal{G}$ the symmetry group supplied to the method.}
  \label{tab:ks_config}
\end{table}

\paragraph{Experimental results} Table~\ref{tab:ks_results} reports the scores against the CTF4Science field. Both SSM-based entries rank near the top of the leaderboard on this first CTF task, with only the reservoir computing method clearly ahead. We note that the equivariant and non-equivariant fits reach comparable accuracy, but, consistent with the parameter-count reductions of the previous examples, the eSSM fit takes $1.48\,$s against $3.69\,$s for SSMLearn. We note in particular, that eSSM clearly outperforms linear methods such as DMD and PyKoopman.

\begin{table}[h!]
    \centering
    \begin{tabular}{l|r|r|r}
        \textbf{Model} & \textbf{Average} & \textbf{E1} & \textbf{E2} \\ \hline
        Reservoir~\cite{jaeger_echo_no_date, maass_computational_2004, pathak_model-free_2018} & 93.21 & 99.97 & 86.45 \\ \hline
        LSTM~\cite{hochreiter1997lstm} & 50.00 & 95.22 & 4.78 \\ \hline
        \textbf{eSSM} & \textbf{48.22} & \textbf{92.85} & \textbf{3.59} \\ \hline
        ODE-LSTM~\cite{coelho2024odelstm} & 47.89 & 80.09 & 15.68 \\ \hline
        \textbf{SSMLearn} (Python) & \textbf{43.56} & \textbf{92.77} & \textbf{-5.65} \\ \hline
        SINDy~\cite{sindy, ensemblesindy} & 41.28 & 84.38 & -1.82 \\ \hline
        Opt DMD~\cite{askham2018variable} & 34.47 & 53.36 & 15.58 \\ \hline
        PyKoopman~\cite{bruntonkutzkoopmanreview22,Pan2024} & 26.59 & 14.60 & 38.58 \\ \hline
        DeepONet~\cite{deeponet} & 23.23 & 36.52 & 9.94 \\ \hline
        KAN~\cite{liu2025kan} & 6.15 & -4.43 & 16.74 \\ \hline
        \textcolor{gray}{Baseline Zeros} & \textcolor{gray}{0.00} & \textcolor{gray}{0.00} & \textcolor{gray}{0.00} \\ \hline
        FNO~\cite{li2021fourier} & -0.50 & 99.00 & -100.00 \\ \hline
        NeuralODE~\cite{chen2018neural} & -22.74 & -36.06 & -9.43 \\ \hline
        Spacetime~\cite{zhang2023spacetime} & -28.25 & 43.49 & -100.00 \\ \hline
        HigherOrder DMD~\cite{LeClainche2017} & -100.00 & -100.00 & -100.00 
    \end{tabular}
    \caption{CTF4Science KS forecasting task (Test 1): short-time score $E_1$,
    long-time score $E_2$ \eqref{eqn:ctf_scores} and their average, ranked by
    average. Bold rows correspond to scores obtained in the present work, the remaining scores are those
    reported in \cite{wyder2026common}.}
    \label{tab:ks_results}
\end{table}\vspace{-0.5cm}

\section{Conclusions}\label{sec:conclusions}

In this work we introduced equivariant spectral submanifold (eSSM) reduction
as a means of computing accurate nonlinear reduced order models of
high-dimensional systems with symmetries. We showed
that SSMs of equivariant systems are themselves equivariant submanifolds, that suitably chosen charts, the reduced
dynamics and the extended normal form all inherit induced actions of the
symmetry group, and we characterised the admissible
Taylor coefficients of equivariant maps. Building on these results, we
developed the eSSM reduction algorithm (Algorithm~\ref{alg:eSSM} and its
discrete-time counterpart, Algorithm~\ref{alg:discrete_eSSM}), whose output is
exactly equivariant by construction, for any input data. In our numerical
experiments the resulting reduction in free parameters translated into
significantly faster fits at matched predictive accuracy, improved fidelity of
the reduced model under symmetry transformations of the data, and competitive
performance on the CTF4Science Kuramoto--Sivashinsky benchmark
\cite{wyder2026common} at a fraction of the computational cost of the
unconstrained method.

\appendix
\section{Description of the method for discrete dynamical systems}\label{app:discrete_dynamical_systems}
The eSSM reduction method extends, with only minor modifications, to discrete dynamical systems of the form
\begin{align}\label{eqn:discrete_system}
  \mathbf{x}_{k+1} = \mathbf{F}(\mathbf{x}_k) = \mathbf{A}\,\mathbf{x}_k + \mathbf{f}(\mathbf{x}_k), \qquad \mathbf{x}_k \in \mathbb{R}^n,
\end{align}
where $\mathbf{f} = \mathcal{O}(\|\mathbf{x}\|^2)$ is smooth and $\mathbf{x}=\mathbf{0}$ is a hyperbolic fixed point, i.e. $\operatorname{Spect}(\mathbf{A})$ does not intersect the unit circle. The discrete system \eqref{eqn:discrete_system} is said to be equivariant with respect to a linear symmetry group $\mathcal{G}$ if
\begin{align*}
  S\mathbf{F}(\mathbf{x}) = \mathbf{F}(S\mathbf{x}), \qquad \forall S\in\mathcal{G}, \mathbf{x} \in \mathbb{R}^n.
\end{align*}
The theory of equivariant spectral submanifolds and the associated reduced dynamics carries over to the discrete setting with only minor straightforward modifications and is therefore, in the interest of brevity, not repeated here. Instead we focus on presenting the main differences in the data-driven eSSM reduction method described in \S\ref{sec:data_driven_ssm_reduction_with_equivariance} and Algorithm~\ref{alg:eSSM} in this discrete setting. Indeed, Step~(i) of Algorithm~\ref{alg:eSSM} only relies on point-values in the observed data and is therefore identical in the discrete setting. The only modifications occur in Step~(ii) of Algorithm~\ref{alg:eSSM}. In the discrete setting the reduced dynamics is the one-step map
\begin{align}\label{eqn:discrete_reduced_dynamics}
  \boldsymbol{\eta}_{k+1} = \mathbf{r}(\boldsymbol{\eta}_k) = \mathbf{V}_1^\top\,\mathbf{F}\big(\mathbf{U}_1\boldsymbol{\eta}_k + \mathbf{h}(\boldsymbol{\eta}_k)\big),
\end{align}
whose linear part $\mathbf{B} = D\mathbf{r}(\mathbf{0})$ is estimated, exactly as in Step~(ii), by the orbit-augmented dynamic mode decomposition
\begin{align*}
  \mathbf{B} = \boldsymbol{\Xi}'_{\mathcal{G}}\boldsymbol{\Xi}_{\mathcal{G}}^{\dagger},
  \qquad
  \boldsymbol{\Xi} = [\boldsymbol{\eta}_1,\dots,\boldsymbol{\eta}_{N-1}],\quad
  \boldsymbol{\Xi}' = [\boldsymbol{\eta}_2,\dots,\boldsymbol{\eta}_{N}].
\end{align*}
However, in contrast to the continuous case, {no matrix logarithm is required}. To derive the extended normal form of the discrete reduced dynamics \eqref{eqn:discrete_reduced_dynamics} we follow the same steps as in \S\ref{sec:computing_ssms}, starting with the eigendecomposition of the linear part $\mathbf{B}$:
\begin{align*}
  \mathbf{B} = \mathbf{W}\boldsymbol{\Lambda}_{\mu}\mathbf{W}^{-1}, \qquad
  \boldsymbol{\Lambda}_{\mu} = \operatorname{diag}(\mu_1,\dots,\mu_d).
\end{align*}
The normal form is now sought as a conjugate one-step map: with the same truncated expansions \eqref{eqn:expansion_for_n}--\eqref{eqn:expansion_for_t}, but with $\boldsymbol{\Lambda}$ replaced by $\boldsymbol{\Lambda}_{\mu}$, we seek
\begin{align*}
  \mathbf{z}_{k+1} = \mathbf{n}(\mathbf{z}_k), \qquad \boldsymbol{\eta} = \mathbf{t}(\mathbf{z}),
\end{align*}
where the differential conjugacy \eqref{eqn:conjugacy} is replaced by its composition form
\begin{align}\label{eqn:conjugacy_discrete}
  \mathbf{t}\big(\mathbf{n}(\mathbf{z})\big) = \mathbf{r}\big(\mathbf{t}(\mathbf{z})\big).
\end{align}
Matching coefficients order by order in \eqref{eqn:conjugacy_discrete} yields the discrete homological equations
\begin{align*}
  \big(\boldsymbol{\mu}^{\mathbf{m}} - \mu_j\big)\,\widetilde{t}_{j,\mathbf{m}} + n_{j,\mathbf{m}} = g_{j,\mathbf{m}},
  \qquad \boldsymbol{\mu}^{\mathbf{m}} := \prod_{l=1}^{d}\mu_l^{m_l},
\end{align*}
in place of \eqref{eqn:homological}. The discrete multipliers can be regarded as the discrete-time analogue of the continuous-time eigenvalues with $\mu_j = e^{\lambda_j\Delta t}$ and, as a result, the additive resonance quantity $\mathbf{m}\cdot\boldsymbol{\lambda} - \lambda_j$ is replaced by its multiplicative counterpart $\boldsymbol{\mu}^{\mathbf{m}} - \mu_j$, and a monomial can be removed from the reduced dynamics precisely when $\boldsymbol{\mu}^{\mathbf{m}} \neq \mu_j$, with small denominators arising whenever $\boldsymbol{\mu}^{\mathbf{m}} \approx \mu_j$. Thus it is natural to characterise the near-resonant monomials in the discrete setting by the index set
\begin{align}\label{eqn:near_resonant_set_discrete}
  \mathcal{I}_\delta := \Big\{ (j,\mathbf{m}) :
  \big|\arg\big(\boldsymbol{\mu}^{\mathbf{m}}/\mu_j\big)\big| \leq \delta
  \Big\}.
\end{align}
We can show analogously to the continuous case (cf. Proposition~\ref{prop:equivariant_normal_form}) that the near-resonant classification \eqref{eqn:near_resonant_set_discrete} is compatible with the symmetry group $\mathcal{G}$, so that the equivariance constraints on the normal form coefficients can be imposed in the same way as in \S\ref{sec:algorithmic_details_essm}. Finally, the coefficient fit of Step~(ii.c) from \S\ref{sec:algorithmic_details_essm} is replaced by the discrete conjugacy fit
\begin{align}\label{eqn:discrete_conjugacy_fit}
  (\mathbf{c}^{\mathbf{n}},\mathbf{c}^{\mathbf{t},\star})
  = \operatorname*{argmin}_{\mathbf{c}^{\mathbf{n}},\,\mathbf{c}^{\mathbf{t},\star}}
  \sum_{i=1}^{N-1}\Big\|\,
  \mathbf{t}^{-1}(\boldsymbol{\eta}_{i+1})
  - \mathbf{n}\big(\mathbf{t}^{-1}(\boldsymbol{\eta}_i)\big)
  \Big\|_2^2.
\end{align}

The discrete-time eSSM reduction method is summarised in Algorithm~\ref{alg:discrete_eSSM}.

\begin{algorithm}
\caption{Discrete-time eSSM reduction method}
\label{alg:discrete_eSSM}
\begin{algorithmic}[1]
\Statex \textbf{Inputs:} as in Algorithm~\ref{alg:eSSM} without $\Delta t$.
\Statex
\Statex \emph{Step (i): identical to Step (i) of Algorithm~\ref{alg:eSSM}.}
\Statex
\Statex \emph{Step (ii): Compute reduced one-step dynamics on $\mathcal{W}(E)$ in extended normal form.}
\State \textbf{Form} the orbit-augmented reduced snapshots
$\boldsymbol{\Xi}_{\mathcal{G}}=[\,S|_E\boldsymbol{\Xi}\,]_{S\in\mathcal{G}}$,
$\boldsymbol{\Xi}'_{\mathcal{G}}=[\,S|_E\boldsymbol{\Xi}'\,]_{S\in\mathcal{G}}$.
\State \textbf{Estimate} the one-step map
$\mathbf{B}=\boldsymbol{\Xi}'_{\mathcal{G}}\boldsymbol{\Xi}_{\mathcal{G}}^{\dagger}$
and \textbf{eigendecompose} $\mathbf{B}=\mathbf{W}\boldsymbol{\Lambda}_{\mu}\mathbf{W}^{-1}$.
\For{$k=2,\dots,M_{\text{ROD}}$}
  \State \textbf{Select} the resonant support
  $\mathcal{I}_\delta=\{(j,\mathbf{m}):|\arg(\boldsymbol{\mu}^{\mathbf{m}}/\mu_j)|\leq\delta\,\Delta t\}$.
  \State \textbf{Compute} the equivariant bases $\mathbf{B}_k^{\mathbf{n}},\mathbf{B}_k^{\mathbf{t}}$
  of $\mathbf{M}_k[\widehat{S}]$ restricted to
  $\operatorname{supp}(\mathcal{I}_\delta),\operatorname{supp}(\mathcal{I}_\delta^{\,c})$.
\EndFor
\State \textbf{Fit} $(\mathbf{c}^{\mathbf{n}},\mathbf{c}^{\mathbf{t},\star})$ on the one-step conjugacy residual \eqref{eqn:discrete_conjugacy_fit}, initialised at $\mathbf{0}$.
\State \textbf{Recover} $\mathbf{c}^{\mathbf{t}}$ by linear least squares and
\textbf{assemble} $\mathbf{N}_k,\mathbf{T}_k,\mathbf{T}_k^{\star}$ via \eqref{eqn:equivariant_basis_N_and_T}.
\Statex
\Statex \textbf{Outputs:} SSM parametrisation
$(\mathbf{U}_1,\mathbf{V}_1,\{\mathbf{W}_k\}_{k=2}^{M},\mathbf{D})$ and reduced one-step dynamics
$(\mathbf{W},\boldsymbol{\Lambda}_{\mu},\{\mathbf{N}_k,\mathbf{T}_k,\mathbf{T}^{\star}_k\}_{k=2}^{M_{\text{ROD}}})$
on $\mathcal{W}(E)$ in extended normal form.
\end{algorithmic}
\end{algorithm}

Using the outputs of Algorithm~\ref{alg:discrete_eSSM}, we can then simulate the
dynamics of \eqref{eqn:discrete_system} similarly to the continuous case, by
iterating the reduced normal form $\mathbf{z}_{k+1} = \mathbf{n}(\mathbf{z}_k)$,
after mapping the initial condition $\mathbf{x}_0$ into normal-form coordinates
via $\mathbf{z}_0 = \mathbf{t}^{-1}(\mathbf{V}_1^\top\mathbf{x}_0)$, and then mapping the trajectory back to the full state space via
$\mathbf{x}_k = \mathbf{U}_1\mathbf{t}(\mathbf{z}_k) + \mathbf{h}\big(\mathbf{t}(\mathbf{z}_k)\big)$.

\section*{Acknowledgments}
The author gratefully acknowledges funding in form of a Henslow Fellowship of the Cambridge Philosophical Society. The author thanks Matt Colbrook (University of Cambridge) for helpful feedback on an early draft of the manuscript.

\bibliographystyle{siamplain}
\bibliography{references}

\end{document}